\documentclass[sigconf]{acmart}

\AtBeginDocument{%
  }

\copyrightyear{2024}
\acmYear{2024}
\setcopyright{rightsretained}
\acmConference[MOBIHOC '24]{The Twenty-fifth International Symposium on Theory, Algorithmic Foundations, and Protocol Design for Mobile Networks and Mobile Computing}{October 14--17, 2024}{Athens, Greece}
\acmBooktitle{The Twenty-fifth International Symposium on Theory, Algorithmic Foundations, and Protocol Design for Mobile Networks and Mobile Computing (MOBIHOC '24), October 14--17, 2024, Athens, Greece}\acmDOI{10.1145/3641512.3686369}
\acmISBN{979-8-4007-0521-2/24/10}

\usepackage[utf8]{inputenc} 
\usepackage[T1]{fontenc}    
\usepackage{url}            
\usepackage{booktabs}       
\usepackage{amsfonts}       
\usepackage{nicefrac}       
\usepackage{microtype}      
\usepackage{xcolor}
\usepackage{verbatim}
\usepackage{graphicx}
\usepackage{textcomp}
\usepackage{cases}
\usepackage{amsmath,amsfonts}
\usepackage{fancyhdr,graphicx}
\usepackage{subfigure}
\usepackage[utf8]{inputenc}
\usepackage{soul}
\usepackage{algorithm, algorithmic}
\usepackage{multirow}
\usepackage{fontenc}
\usepackage{amsthm}
\usepackage{enumitem}
\newtheorem{theorem}{Theorem}
\newtheorem{lemma}{Lemma}

\newtheorem{definition}{Definition}
\newtheorem{remark}{Remark}

\begin{document}

\title{Distributed No-Regret Learning for Multi-Stage Systems with End-to-End Bandit Feedback}

\author{I-Hong Hou}
\email{ihou@tamu.edu}
\affiliation{%
  \institution{Department of ECE, Texas A\& M University}
  \city{College Station}
  \state{Texas}
  \country{USA}
}

\begin{abstract}
This paper studies multi-stage systems with end-to-end bandit feedback. In such systems, each job needs to go through multiple stages, each managed by a different agent, before generating an outcome. Each agent can only control its own action and learn the final outcome of the job. It has neither knowledge nor control on actions taken by agents in the next stage. The goal of this paper is to develop distributed online learning algorithms that achieve sublinear regret in adversarial environments. 

The setting of this paper significantly expands the traditional multi-armed bandit problem, which considers only one agent and one stage. In addition to the exploration-exploitation dilemma in the traditional multi-armed bandit problem, we show that the consideration of multiple stages introduces a third component, education, where an agent needs to choose its actions to facilitate the learning of agents in the next stage. To solve this newly introduced exploration-exploitation-education trilemma, we propose a simple distributed online learning algorithm, $\epsilon-$EXP3. We theoretically prove that the $\epsilon-$EXP3 algorithm is a no-regret policy that achieves sublinear regret. Simulation results show that the $\epsilon-$EXP3 algorithm significantly outperforms existing no-regret online learning algorithms for the traditional multi-armed bandit problem. 
\end{abstract}
\maketitle

\section{Introduction} \label{section:intro}

In many modern applications, a job consists of multiple stages that need to be performed by different agents, and the decision made in each stage can impact the performance of the job. For example, consider a mobile edge computing application where a mobile user offloads video analytic jobs to nearby edge servers, and each edge server is equipped with multiple video analytic neural networks with different precision and latency. To process a video analytic job, the mobile user needs to first decide which edge server to forward this job to. After the mobile user forwards the job to an edge server, the edge server needs to decide which neural network to employ for this job. The performance of the job depends on the accuracy of the result and the end-to-end latency, which includes both communication and computation delays. As another example, consider packet deliveries in multi-hop networks consisting of multiple routers. Upon receiving a packet, a router needs to decide which router to forward the packet to. The performance of the packet depends on the end-to-end latency.

This paper studies the problem of designing distributed online learning algorithms under which all agents jointly learn the optimal decisions with minimum coordination, even when the outcomes of decisions are determined by an adversary. Developing such algorithms are challenging due to three major reasons. First, in most computer and network systems, it is desirable to employ distributed algorithms where each agent can only make decisions of its own action and has neither knowledge nor control on the actions taken by agents in the next stages. Second, in many systems, an agent can only observe the end-to-end outcome of the joint effects of all stages, but cannot know how actions taken in each individual stage contribute to the end-to-end outcome. 
Finally, an agent can only learn the outcome of its chosen action, which is typically referred to as the \emph{bandit feedback} in the literature.


We note that the traditional multi-armed bandit problem is a special case of multi-stage systems when there is only one stage. The main challenge of the traditional multi-armed bandit problem is to balance between learning the outcomes of each possible action (exploration) and choosing the action with the best historic outcomes (exploitation).  The general problem of multi-stage systems is even more challenging because agents in the next stage are also learning agents and their ability to learn depends on actions taken by the agent in the previous stage. In the example of mobile edge computing, an edge server can only process a job and learn the outcome when it receives a job from the mobile user. When a mobile user receives a poor outcome from an edge server, it may be because the edge server has yet to learn the optimal action and chooses a bad neural network, rather than because the edge server has no good options. To ensure that all edge servers can learn the optimal actions, the mobile user needs to \emph{educate} edge servers by forwarding a sufficient number of jobs to each of them. Thus, the mobile user is facing an exploration-exploitation-education trilemma.

To study the online learning problem in multi-stage systems, we propose an analytical model that captures both the distributed decision making and the end-to-end bandit feedback. We first consider the simplified case when each agent can observe the outcomes of all its actions, including those not taken. We show that we can achieve sublinear regret by making all agents employ the Normalized Exponential Gradient (normalized-EG) algorithm independently in a distributed fashion.

Next, we study the multi-stage system with only end-to-end bandit feedback, that is, an agent can only observe an outcome if it receives a job and it can only observe the outcome of its chosen action. To address the exploration-exploitation-education trilemma, we propose a simple distributed online learning algorithm called $\epsilon-$EXP3. The $\epsilon-$EXP3 algorithm has two operation modes, a uniform selection mode in which the agent chooses actions uniformly at random to provide equal education to agents in the next stage, and an EXP3 mode where the agent employs a variation of the EXP3 algorithm to balance the tradeoff between exploration and exploitation. By randomly alternating between these two modes, the $\epsilon-$EXP3 algorithm explicitly address all three of exploration, exploitation, and education. We theoretically prove that, when applying $\epsilon-$EXP3 on a system with $L$ stages, the regret accumulated over $T$ rounds is at most $O(T^{\frac{L}{L+1}})=o(T)$.

To understand the fundamental regret lower bounds and the role of education in multi-stage systems, we study a class of time-homogeneous oracle policies. These policies assume that each node can know the outcome of all actions \emph{before} making a decision. Therefore, there is no need to explore and each node only faces an education-exploitation dilemma. We show that the regret of these policies is at least $\Theta(T^{\frac{L-1}{L}})$, which is only slightly better than the regret of $\epsilon-$EXP3.

The utility of the $\epsilon-$EXP3 algorithm is further evaluated by simulations. The simulation results show that the regret of the $\epsilon-$EXP3 algorithm indeed scales as $O(T^{\frac{L}{L+1}})$. We also evaluate two other policies that are no-regret policies for the traditional one-stage bandit problem. Surprisingly, we show that their regrets scale as $\theta(T)$ even when the system has only two stages. The simulation results demonstrate that the education component is indeed critical in multi-stage systems.

The rest of the paper is organized as follows. Section~\ref{section:related} surveys existing studies on adversarial bandit problems, mobile edge computing, and multi-hop networks. Section~\ref{section:model} introduces our system model and problem definition. Section~\ref{section:omd} studies the simplified case with complete one-hop feedback. Section~\ref{section:exp3} introduces and analyzes the $\epsilon-$EXP3 algorithm for systems with end-to-end bandit feedback. Section~\ref{section:lowerbound} establishes a regret lower bound for time-homogeneous oracle policies. Section~\ref{section:simulation} presents our simulation results. Finally, Section~\ref{section:conclusion} concludes the paper.
\section{Related Work}  \label{section:related}

\textbf{No-regret bandit learning.} 
The multi-armed bandit problem has attracted significant research interests because it elegantly captures the trade-off between exploration and exploitation. In adversarial environments, the celebrated EXP3 algorithm has been proved to achieve a regret bound of $O(T^{\frac{1}{2}})$ \cite{auer2002nonstochastic}. This bound has later been shown to be tight \cite{audibert2009minimax}. There have been many studies on variations and improvements of the EXP3 algorithm \cite{audibert2010regret, neu2015explore, uchiya2010algorithms, wei2018more}. All these studies only consider systems with one agent.

There have been considerable recent efforts on cooperative learning \cite{cesa2020cooperative, chawla2020gossiping, newton2022asymptotic, agarwal2022multi, landgren2021distributed, madhushani2021one, martinez2019decentralized, liu2010distributed} where agents help each other find the optimal action. These studies assume that the reward of an agent only depends on the action of that agent. In contrast, our work allows different agents to have different sets of actions and considers that the reward in each round depends on the actions of all agents. Singla, Hassani, and Krause \cite{singla2018learning} has studied a distributed learning problem in two-stage systems. It is limited to the special case of two stages and requires the root node to have the ability to block feedback information.


\noindent\textbf{Mobile edge computing.} 
One emerging application of mobile edge computing is cloud/edge robotics where a robot offloads its computation tasks to nearby edge servers. 
An important challenge for cloud/edge robotics is that the performance of a job depends on both the quality of the outcome and the end-to-end delay. To enable flexible trade-off between quality and latency, Jiang et al. \cite{jiang2018chameleon} has proposed a controller that dynamically select the suitable neural network configuration. Wu et al. \cite{wu2021soudain} has modeled the problem of adaptive configuration as an integer programming problem and proposed a heuristic for it. He et al. \cite{he2022adaconfigure} has employed a reinforcement learning approach for adaptive configuration. Zhang et al. \cite{zhang2021adaptive} has employed Lyapunov optimization to learn the optimal configuration over time. These studies only study the decisions of edge servers and they only consider stationary systems. Chinchali et al. \cite{chinchali2021network} has proposed using deep reinforcement learning for the offloading decisions of robots, but it does not consider the adaptive configuration of edge servers. To the best of our knowledge, no existing work has jointly optimized the offloading decision of robots and the adaptive configuration of edge servers in unknown and time-varying environments.

\noindent\textbf{Multi-hop networks.} There have been significant interests in employing online learning or reinforcement learning techniques for multi-hop networks, but few of them have been able to characterize end-to-end delay and enforcing end-to-end deadline. Bhorkar and Javidi \cite{bhorkar2010no} has proposed a no-regret learning policy for minimizing end-to-end transmission cost. Park, Kang, and Joo \cite{park2021learning} has proposed a UCB-based algorithm for throughput-optimality in multi-hop wireless networks. Al Islam et al. \cite{al2012multi} has considered the problem of end-to-end congestion control problem in multi-hop networks as a multi-armed bandit problem. Zhang, Tang, and Wang \cite{zhang2020cooperative} has studied the problem of relay selection to minimize energy consumption in two-hop networks. None of the aforementioned studies consider end-to-end delay or end-to-end deadline.

Mao, Koksal, and Shroff \cite{mao2014optimal}, Deng, Zhao, and Hou \cite{deng2019online}, and Gu, Liu, Shen \cite{gu2021asymptotically} have all studied online scheduling and routing algorithms for multi-hop networks with end-to-end deadlines, but they require precise knowledge on the capacity and latency of each link. HasanzadeZonuzy, Kalathil, and Shakkottai \cite{hasanzadezonuzy2020reinforcement} has proposed a model-based reinforcement learning algorithm for real-time multi-hop networks but it only works for stationary systems. Both Lin and van der Schaar \cite{lin2010autonomic} and Shiang, and van der Schaar \cite{shiang2010online} employ reinforcement learning to serve delay-sensitive traffic by modeling multi-hop networks as stationary MDPs with unknown kernels. To the best of our knowledge, no existing work has studied the regret of delay-sensitive multi-hop networks in adversarial environments.

\section{System Model}
\label{section:model}

We represent a multi-stage system as a tree with depth $L+1$. We denote the root node by $r$ and the set of leaf nodes by $\mathcal{L}$. We use $\mathcal{C}_i$ to denote the set of children of a non-leaf node $i$. 
In each round $t$, the root node $r$ receives a job. It selects a child node $f[r,t]\in \mathcal{C}_r$, possibly at random, and forwards the job to it. Likewise, every non-leaf node $i$ randomly selects a child node $f[i,t]\in \mathcal{C}_i$ and, if $i$ receives a job in round $t$, forwards the job to $f[i,t]$. When the job reaches a leaf node $j$, it generates a cost of $c[j,t]\in [0,1]$. The value of $c[j,t]$ is revealed to all nodes between the root and the leaf node $j$ through an end-to-end feedback message.

We note that each node only has limited feedback information in this setting. In particular, if a node receives a job in round $t$, then it will only know its own choice and the final cost. It has neither knowledge nor control on the choices made by its children. This is to reduce coordination overhead and to protect privacy. If a node does not receive a job in a round, then it will not receive any feedback information.

To see how our model can be used to capture mobile edge computing, we can consider the example shown in Fig.~\ref{fig:ex_mec_system}. In this system, a robot chooses one of two edge servers to offload its video analytic jobs. Each edge server has two neural networks to choose from. This system can be modeled as a tree with $L=2$ as shown in Fig.~\ref{fig:ex_mec_tree}. In Fig.~\ref{fig:ex_mec_tree}, the robot is the root that chooses between child $A$ and child $B$. Each of these child nodes corresponds to an edge server, and each child node chooses between two leaf nodes. Each leaf node is labeled by $X:n$, where $X$ indicates the edge server chosen by the robot, and $n$ indicates the neural network chosen by the edge server. The cost of a leaf node is chosen to reflect the delay and the quality of the outcome of the video analytic job.

\begin{figure}[h]
    \centering
    \subfigure[System illustration]
    {
    \includegraphics[width=0.7\linewidth]{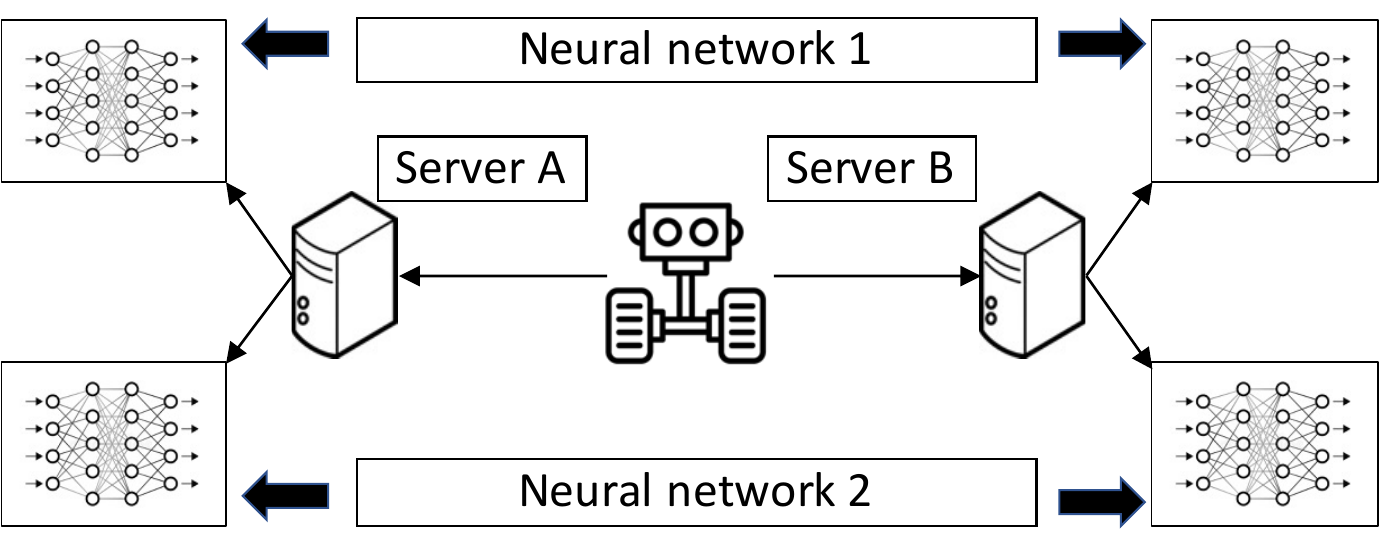}
    \label{fig:ex_mec_system}
    }
    \subfigure[Tree model]
    {
    \includegraphics[width=0.50\linewidth]{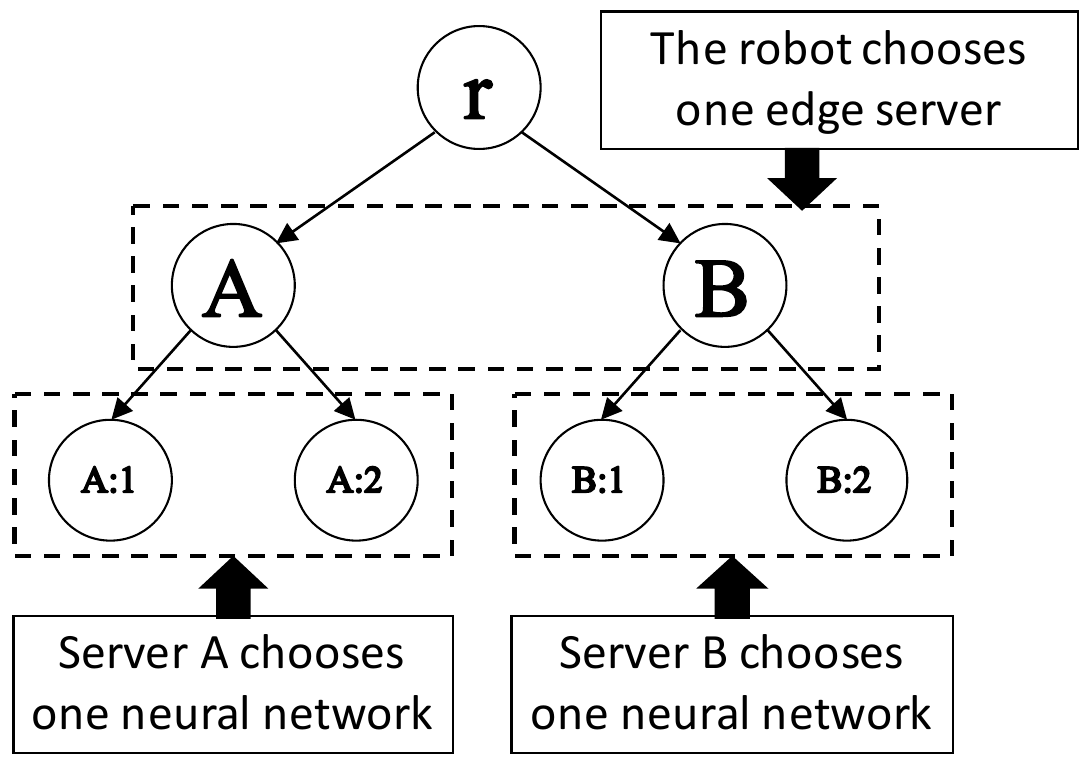}
    \label{fig:ex_mec_tree}
    }
    \caption{A mobile edge computing system and its tree model}
    \label{fig:ex_mec}
\end{figure}

This model can also be used to capture multi-hop networks. In multi-hop networks, the root is the source that generates packets to be delivered. Each non-leaf node corresponds to the path used to transfer a packet to an intermediate router. The choice of that non-leaf node corresponds to choosing the next-hop by the intermediate router. Each leaf node is a complete path from the source to the destination and its cost can be chosen to reflect end-to-end delay. We note that we do not require the topology of the multi-hop networks to be trees. Even when the topology of a network is not a tree, the set of all loop-free paths from the source to the destination can still be represented as a tree. 

Each non-leaf node $i$ employs a distributed online policy that determines the probability of forwarding a job to a child node $j$ in round $t$, denoted by $x[i,j,t]:=Prob(f[i,t]=j)$, in the event that $i$ receives a job. We have $x[i,j,t]\geq 0$ and $\sum_{j\in \mathcal{C}_i}x[i,j,t]=1$. Node $i$ needs to determine the values of $x[i,j,t]$ using only the information available up to round $t-1$.

We now characterize the performance of a distributed online policy after it determines the values of $x[i,j,t]$ and selects $f[i,t]$ accordingly. 
Let $y[i,t]$ be the random variable indicating the amount of cost that would be incurred if node $i$ receives a job in round $t$, under the probability distribution of $f[i,t]$. By definition, we have $y[j,t]=c[j,t]$ for each leaf node $j\in \mathcal{L}$. For each non-leaf node $i$, $y[i,t]$ can be calculated recursively through $y[i,t]=y[f[i,t],t]$. Also, let $w[i,t]:=E\Big[y[i,t]\Big|\mathcal{H}_{t-1}\Big]$ be the conditional expected amount of cost incurred if node $i$ receives a job in round $t$, given all events up to round $t-1$, denoted by $\mathcal{H}_{t-1}$. The value of $w[i,t]$ can then be calculated recursively by $w[j,t]=c[j,t]$ for each leaf node $j$ and $w[i,t] = \sum_{j\in \mathcal{C}_i}x[i,j,t]w[j,t]$ for each non-leaf node $i$. The total expected cost incurred by the distributed online policy over a time horizon of $T$ rounds can then be written as $\sum_{t=1}^T E\Big[y[r,t]\Big]=\sum_{t=1}^T E\Big[w[r,t]\Big]$.

We compare the cost of a distributed online policy against a stationary policy where each node selects the same child node in each round $t$, i.e., $f[i,t]\equiv f_i, \forall t$. Under a stationary policy, all jobs will reach the same leaf node $j^*$, and hence the total cost incurred by the stationary policy is $\sum_{t=1}^Tc[j^*,t]$. The optimal stationary policy is the one that has a minimum cost among all stationary policies and its cost is $\min_{j\in\mathcal{L}} \sum_{t=1}^Tc[j,t]$. We therefore define the regret of a distributed online policy as $\sum_{t=1}^T E\Big[y[r,t]\Big]-\min_{j\in\mathcal{L}} \sum_{t=1}^Tc[j,t]$. Our goal is to design a \emph{no-regret} policy whose regret is sublinear in $T$ under all possible vectors of $c[j,t]$:

\begin{definition}
A distributed online policy is said to be a no-regret policy if $\sum_{t=1}^T E\Big[y[r,t]\Big]-\min_{j\in\mathcal{L}} \sum_{t=1}^Tc[j,t]=o(T)$.
\end{definition}


\section{Preliminary: Policy with Complete One-Hop Feedback}
\label{section:omd}

In this section, we first study the simplified case where each non-leaf node has complete one-hop feedback from its children. Specifically, each non-leaf node $i$, regardless whether it receives a job or not, will be able to learn the values of $y[j,t]$ for each of its children $j\in \mathcal{C}_i$ after $i$ chooses $f[i,t]$. Node $i$ can then use these values to update the values of $x[i,j,t+1]$. We emphasize that the communication overhead between a child node $j$ and its parent node contains only one single scalar $y[j,t]$ in each round. Hence, the feedback information that a non-leaf node has is still limited. For example, a non-leaf node has neither knowledge nor control over actions taken by its children.

We consider that each non-leaf node $i$ independently employs the Normalized Exponential Gradient (normalized-EG) algorithm, a special case of the Online Mirror Descent algorithm and the Follow-the-Regularized-Leader algorithm. Under the normalized-EG algorithm, each non-leaf node $i$ maintains a variable $\theta[i,j,t]$ for each $j\in\mathcal{C}_i$ by setting $\theta[i,j,1]=0$ and $\theta[i,j,t]=\theta[i,j,t-1]-y[j,t-1]$ for all $t>1$. It then chooses $x[i,j,t]=\frac{e^{\eta_i\theta[i,j,t]}}{\sum_{k\in\mathcal{C}_i}e^{\eta_i\theta[i,k,t]}}$ in each round $t$, where $\eta_i$ is a constant whose value will be determined later. The normalized-EG algorithm is an online policy because $\theta[i,j,t]$ can be calculated only based on $y[j,1], y[j,2], \dots, y[j,t-1]$. A formal description of the normalized-EG algorithm is presented in Alg.~\ref{alg:normalized-EG}.

\begin{algorithm}[h]
   \caption{Distributed Normalized Exponential Gradient}
   \label{alg:normalized-EG}
\begin{algorithmic}[1]
\STATE $\eta_i\leftarrow$ a pre-determined constant 
\STATE $\theta[i,j]\leftarrow0, \forall j\in\mathcal{C}_i$
\FOR{each round $t$}
\STATE $x[i,j]\leftarrow\frac{e^{\eta_i\theta[i,j]}}{\sum_{k\in\mathcal{C}_i}e^{\eta_i\theta[i,k]}}, \forall j\in\mathcal{C}_i$  
\STATE Select a child $f[i]$ with $Prob(f[i]=j)=x[i,j]$ 
\FOR{each $j\in\mathcal{C}_i$} 
\STATE Obtain $y[j]$ from child $j$
\STATE $\theta[i,j]\leftarrow \theta[i,j]-y[j]$
\ENDFOR
\STATE $y[i]\leftarrow y[f[i]]$ 
\STATE Report $y[i]$ to the parent node
\ENDFOR
\end{algorithmic}
\end{algorithm}

The regret of the normalized-EG algorithm has been extensively studied for the special case when $L=1$. We will further show that the normalized-EG algorithm is a no-regret policy for the general case $L>1$. It is important to note that the values of $y[j,t]$ observed by $i$ under the normalized-EG algorithm can be different from those under the optimal stationary policy. This is because the values of $y[j,t]$ depend on the decisions made by children nodes $j$. To distinguish between these two policies, we let $y_n[j,t]$ be the values of $y[j,t]$ under the normalized-EG algorithm and let $y_*[j,t]$ be those under the optimal stationary policy. 

Since the normalized-EG algorithm is updated with respect to $y_n[j,t]$, we let $\mathcal{Y}_n[i,t]:=\{y_n[j,\tau],\forall j\in \mathcal{C}_i, \tau\in[1,t]\}$ be the sequences of costs of all children of $i$ up to round $t$ and have the following from existing studies:

\begin{lemma}[\cite{shalev2012online}, Theorem 2.22]   \label{lemma:omd}
If $y_n[j,\tau]\geq0$ for all $j\in \mathcal{C}_i$ and $\tau\in[1, T]$, then the expected total cost incurred by $i$ given $\mathcal{Y}_n[i]$ is upper-bounded by:
\begin{align}
    \sum_{t=1}^T E\Big[y_n[i,t]\Big|\mathcal{Y}_n[i,t]\Big]
    \leq& \min_{j\in\mathcal{C}_i}\sum_{t=1}^Ty_n[j,t]+\frac{\log |\mathcal{C}_i|}{\eta_i}\nonumber\\
    &+\eta_i\sum_{t=1}^T\sum_{j\in\mathcal{C}_i}x[i,j,t]{y_n[j,t]}^2.\nonumber
\end{align}
Moreover, if $y_n[j,\tau]\in[0,1],\forall j\in \mathcal{C}_i, \tau\in[1,T]$, then setting $\eta_i=\sqrt{\frac{\log |\mathcal{C}_i|}{T}}$ yields:
\begin{align}
    &\sum_{t=1}^T E\Big[y_n[i,t]\Big|\mathcal{Y}_n[i,t]\Big]
    \leq \min_{j\in\mathcal{C}_i}\sum_{t=1}^Ty_n[j,t]+2\sqrt{T\log |\mathcal{C}_i|}.\nonumber
\end{align}
$\Box$
\end{lemma}

Under the optimal stationary policy, each node will choose to forward the job to the child that incurs the least cost through all $T$ rounds. Hence, we have $\sum_{t=1}^T y_*[i,t]=\min_{j\in\mathcal{C}_i}\sum_{t=1}^Ty_*[j,t].$ We now prove that the normalized-EG algorithm is still a no-regret policy for the multi-stage system:

\begin{theorem} \label{theorem:regret of normalized-EG}
If each non-leaf node has as most $D$ children, then, by setting $\eta_i=\sqrt{\frac{\log |\mathcal{C}_i|}{T}}, \forall i$, the expected cost incurred by the root node $r$ is upper-bounded by:
\begin{equation}
    \sum_{t=1}^T E\Big[y_n[r,t]\Big]\leq \min_{j\in\mathcal{L}}\sum_{t=1}^Tc[j,t]+2L\sqrt{T\log D}.
\end{equation}
\end{theorem}
\begin{proof}
We will prove the theorem by establishing the following statement by induction: If a node $i$ is $(L-h)$-hops from the root node $r$, then $\sum_{t=1}^T E\Big[y_n[i,t]\Big]\leq \sum_{t=1}^Ty_*[i,t]+2h\sqrt{T\log D}$. Please see Appendix~\ref{appendix:thm1} for details.
\end{proof}

\section{Policy with End-to-End Bandit Feedback}
\label{section:exp3}

In this section, we consider the case where each non-leaf node only has bandit feedback. Specifically, if a node does not receive a job in round $t$, then it will not get any feedback. If a node receives a job and forwards it to a child node $j=f[i,t]$, then it will only learn the value of $y[j,t]$. As discussed in earlier sections, online policies with end-to-end bandit feedback faces a trilemma between exploration, i.e., choosing a child to learn its cost, exploitation, i.e., choosing a child to incur low cost, and education, i.e., choosing a child so that it has a chance to learn and improve its policy.

We propose a simple distributed online learning policy to address the exploration-exploitation-education trilemma called the $\epsilon-$EXP3 algorithm. Under the $\epsilon-$EXP3 algorithm, each non-leaf node $i$ maintains a variable $\theta[i,j,t]$ for each $j\in\mathcal{C}_i$, which it will use to determine $x[i,j,t]$. When a node $i$ sends a job to a child node $j=f[i,t]$, node $i$ also includes a variable $v[j,t]$ indicating the probability that the child node $j$ receives a job in round $t$. Since a node $j$ will receive a job if its parent node $i$ receives a job and node $i$ chooses $j$, the value of $v[j,t]$ can be calculated by $v[j,t]=v[i,t]x[i,j,t]$.

We now discuss how a non-leaf node $i$ decides $f[i,t]$ in each round $t$. There are two modes for choosing $f[i,t]$ and node $i$ randomly decides which mode to operate in in each $t$. Each node $i$ is assigned two pre-determined constants $\epsilon_i$ and $\eta_i$. With probability $\epsilon_i$, node $i$ is in the \emph{uniform selection mode} and it chooses $f[i,t]$ uniformly at random from its children, that is, $Prob(f[i,t]=j)=1/|\mathcal{C}_i|, \forall j\in \mathcal{C}_i$. With probability $1-\epsilon_i$, node $i$ is in the \emph{EXP3 mode} and it chooses $f[i,t]=j$ with probability $\frac{e^{\eta_i\theta[i,j,t]}}{\sum_{k\in\mathcal{C}_i}e^{\eta_i\theta[i,k,t]}}$. We use $m[i,t]\in\{U,E\}$ to denote the mode of node $i$, where $U$ is the uniform selection mode and $E$ is the EXP3 mode. Combining these two modes and we have $x[i,j,t]=\epsilon_i \frac{1}{|\mathcal{C}_i|}+(1-\epsilon_i)\frac{e^{\eta_i\theta[i,j,t]}}{\sum_{k\in\mathcal{C}_i}e^{\eta_i\theta[i,k,t]}}$. 

After choosing $f[i,t]$ for each node $i$, we can set $y_\epsilon[i,t]=c[i,t]$ for each leaf node and set $y_\epsilon[i,t]=y_\epsilon[f[i,t],t]$ for each non-leaf node, where the subscript $\epsilon$ is to highlight that this corresponds to the values of $y[j,t]$ under the $\epsilon-$EXP3 algorithm. We note that, even if node $i$ does not receive a job in round $t$, the value of $y_\epsilon[i,t]$ is still well-defined, but node $i$ does not know its value.

Finally, we discuss how node $i$ determines $\theta[i,j,t]$. Node $i$ initializes $\theta[i,j,1]=0$ for all children $j$. If node $i$ receives a job in round $t$, then it learns the value of $y_\epsilon[f[i,t],t]$. Node $i$ sets $z[f[i,t],t]=\frac{y_\epsilon[f[i,t],t]|\mathcal{C}_i|}{v[i,t]}$, if $m[i,t]=U$, and sets $z[f[i,t],t]=\frac{y_\epsilon[f[i,t],t]\sum_{k\in\mathcal{C}_i}e^{\eta_i\theta[i,k,t]}}{v[i,t]e^{\eta_i\theta[i,j,t]}}$, if $m[i,t]=E$. Node $i$ sets $z[j,t]=0$ for all $j\neq f[i,t]$. On the other hand, if node $i$ does not receive a job in round $t$, then it sets $z[j,t]=0,\forall j\in\mathcal{C}_i$. Finally, it sets $\theta[i,j,t+1]=\theta[i,j,t]-z[j,t], \forall j\in\mathcal{C}_i$.

Alg.~\ref{alg:exp3} describes the $\epsilon-$EXP3 algorithm in detail, where we streamline some of the steps for easier implementation.

\begin{algorithm}[h]
   \caption{$\epsilon$-EXP3}
   \label{alg:exp3}
\begin{algorithmic}[1]
\STATE $\eta_i, \epsilon_i\leftarrow$ pre-determined constants 
\STATE $\theta[i,j]\leftarrow0, \forall j\in\mathcal{C}_i$
\FOR{each round $t$}
\IF{Node $i$ receives a job and $v[i,t]$ from its parent} 
\STATE $x[i,j]\leftarrow \epsilon_i \frac{1}{|\mathcal{C}_i|}+(1-\epsilon_i)\frac{e^{\eta_i\theta[i,j]}}{\sum_{k\in\mathcal{C}_i}e^{\eta_i\theta[i,k]}}, \forall j\in\mathcal{C}_i$
\STATE $v[j,t]\leftarrow v[i,t]x[i,j], \forall j\in\mathcal{C}_i$
\STATE Randomly select $m[i]\in \{U,E\}$ with $Prob(m[i] = U)=\epsilon_i$
\IF{$m[i] = U$}   
\STATE Select a child $f[i]\in\mathcal{C}_i$ uniformly at random
\STATE Forward the job and $v[f[i],t]$ to child $f[i]$ and obtain $y_\epsilon[f[i]]$ from $f[i]$
\STATE $\theta[i, f[i]]\leftarrow \theta[i, f[i]]-\frac{y_\epsilon[f[i]]|\mathcal{C}_i|}{v[i,t]}$
\STATE Return $y_\epsilon[i]\leftarrow y_\epsilon[f[i]]$ to the parent
\ELSE 
\STATE Select a child $f[i]\in\mathcal{C}_i$ with $Prob(f[i]=j)=\frac{e^{\eta_i\theta[i,j]}}{\sum_{k\in\mathcal{C}_i}e^{\eta_i\theta[i,k]}}$
\STATE Forward the job and $v[f[i],t]$ to child $f[i]$ and obtain $y_\epsilon[f[i]]$ from $f[i]$
\STATE $\theta[i, f[i]]\leftarrow \theta[i, f[i]]-\frac{y_\epsilon[f[i]]\sum_{k\in\mathcal{C}_i}e^{\eta_i\theta[i,k]}}{v[i,t]e^{\eta_i\theta[i,j]}}$
\STATE Return $y_\epsilon[i]\leftarrow y_\epsilon[f[i]]$ to the parent
\ENDIF
\ENDIF
\ENDFOR
\end{algorithmic}
\end{algorithm}

\begin{remark}
The reason that the $\epsilon-$EXP3 algorithm has two different modes to choose $f[i,t]$ is to address the exploration-exploitation-education trilemma. When node $i$ is in the uniform selection mode, its goal is to provide equal education to all its children. Hence, it selects $f[i,t]$ uniformly at random so that each child node has the same chance of receiving a job and learning from its outcome. When node $i$ is in the EXP3 mode, its goal is to balance the trade-off between exploration and exploitation. Hence, it employs a very similar way of choosing $f[i,t]$ as the EXP3 algorithm. The value of $\epsilon_i$ determines the portion of time that node $i$ dedicate to education. On the other hand, the value of $\eta_i$ determines the trade-off between exploration and exploitation when node $i$ is in the EXP3 mode, where larger $\eta_i$ means more emphasis on exploitation. The values of $\epsilon_i$ and $\eta_i$ will be determined later.
\end{remark}

We now analyze the regret of $\epsilon-$EXP3. Our first step is to establish some properties of $z[j,t]$. We let $\mathcal{Y}_\epsilon[i,t]:=\{y_\epsilon[j,\tau],\forall j\in \mathcal{C}_i, \tau\in[1,t]\}$ be the sequences of costs of all children of $i$ up to round $t$ and let $\mathcal{Z}[i,t]:=\{z[j,\tau],\forall j\in \mathcal{C}_i, \tau\in[1,t]\}$ be all the values of $z[j,\tau]$ that has been observed by $i$ up to round $t$. We then have the following:
\begin{lemma}\label{lemma:z[j,t]}
For any non-leaf node $i$,
\begin{equation}
    E\Big[z[j,t]\Big|\mathcal{Y}_\epsilon[i,t],\mathcal{Z}[i,t-1]\Big]=y_\epsilon[j,t],\label{equation:z[j,t]}
\end{equation}
and
\begin{align}
    &E\Big[z[j,t]^2\Big|\mathcal{Y}_\epsilon[i,t],\mathcal{Z}[i,t-1]\Big]\nonumber
    \\
    =&\Big(\epsilon_i|\mathcal{C}_i|+(1-\epsilon_i) \frac{\sum_{k\in\mathcal{C}_i}e^{\eta_i\theta[i,k,t]}}{e^{\eta_i\theta[i,j,t]}}\Big)\frac{y_\epsilon[j,t]^2}{v[i,t]}.    \label{equation:z[j,t]^2}
\end{align}
\end{lemma}
\begin{proof}
Please see Appendix~\ref{appendix:lemma2}. 
\end{proof}

Next, we show that, if node $i$ is in the EXP3 mode at round $t$, then its expected cost is the same as the expected cost of running the normalized-EG algorithm against the sequence $z[j,t]$.

\begin{lemma} \label{lemma:y_n vs y_e}
By considering a sequence $y_n[j,\tau]=z[j,\tau], \forall j\in \mathcal{C}_i, \tau\in[1,T]$ for the normalized-EG algorithm,
\begin{align}
    &E\Big[y_\epsilon[i,t]\Big|m[i,t]=E,\mathcal{Y}_\epsilon[i,t],\mathcal{Z}[i,t-1]\Big]\nonumber\\
    =&E\Big[E\Big[y_n[i,t]\Big|\mathcal{Y}_n[i,t]=\mathcal{Z}[i,t]\Big]\Big],\nonumber
\end{align}
where the outer expectation on the right hand side is taken with respect to $z[j,t]$.
\end{lemma}
\begin{proof}
Please see Appendix~\ref{appendix:lemma3}. 
\end{proof}

Our next step is to bound the difference between $\sum_{t=1}^TE\Big[y_\epsilon[i,t]\Big]$ and $\min_{j\in\mathcal{C}_i}\sum_{t=1}^Ty_\epsilon[j,t]$ under any given given sequence of\\ $y_\epsilon[j,1],\dots, y_\epsilon[j,T]$, for all $j\in\mathcal{C}_i$.

\begin{lemma} \label{lemma:exp3-one-step}
If each non-leaf node has at most $D$ children, then
\begin{align}
    &\sum_{t=1}^TE\Big[y_\epsilon[i,t]\Big|\mathcal{Y}_\epsilon[i,t]\Big]\nonumber\\
    \leq&\min_{j\in\mathcal{C}_i}\sum_{t=1}^Ty_\epsilon[j,t]+\epsilon_i T+\frac{\log D}{\eta_i}+\eta_i{\sum_{t=1}^T\frac{D}{v[i,t]}},\nonumber
\end{align}
for all non-leaf node $i$. Moreover, if the depth of the tree is $L+1$, then setting $\eta_i=T^{-\frac{L}{L+1}}$ for all $i$ and setting $\epsilon_i$ to be 0 if $C_i\subset\mathcal{L}$, and $DT^{-\frac{1}{L+1}}$ otherwise yields
\begin{align}
    &\sum_{t=1}^TE\Big[y_\epsilon[i,t]\Big|\mathcal{Y}_\epsilon[i,t]\Big]-\min_{j\in\mathcal{C}_i}\sum_{t=1}^Ty_\epsilon[j,t]\nonumber\\
    \leq& \begin{cases}
    (D+\log D)T^{\frac{L}{L+1}}, &\textbf{if $C_i\subset\mathcal{L}$},\\
    (2D+\log D)T^{\frac{L}{L+1}} &\textbf{else}.
    \end{cases}\nonumber
\end{align}
\end{lemma}
\begin{proof}
Please see Appendix~\ref{appendix:lemma4}. 
\end{proof}

\begin{remark}
An explanation for the choice of $\epsilon_i$ is in order. We set $\epsilon_i=0$ if all children of node $i$ are leaf nodes. Since leaf nodes do not have any children to choose from, they have nothing to learn and do not need education. Hence, node $i$ can operate exclusively in the EXP3 mode. On the other hand, if node $i$ has some children that are non-leaf nodes, then node $i$ needs to educate these children. Hence, it operates in the uniform selection mode with a constant probability.
\end{remark}

We will now prove that the $\epsilon-$EXP3 policy is a no-regret policy.

\begin{theorem} \label{theorem:eps-EXP3 regret}
If the depth of the tree is $L+1$ and each non-leaf node $i$ has at most $D$ children, then, by using the same settings of $\eta_i$ and $\epsilon_i$ as in Lemma~\ref{lemma:exp3-one-step}, the regret of $\epsilon$-EXP3 is at most $((2L-1)D+L\log D)T^{\frac{L}{L+1}}=o(T)$.
\end{theorem}
\begin{proof}
We will prove the theorem by establishing the following statement: If a node $i$ is $(L-h)$-hops from the root node $r$, then $\sum_{t=1}^T E\Big[y_n[i,t]\Big]\leq \sum_{t=1}^Ty_*[i,t]+((2h-1)D+h\log D)T^{\frac{L}{L+1}}$, where $y_*[i,t]$ is the cost under the optimal stationary policy.

We prove the statement by induction. First, consider the case $h=1$, that is, the node $i$ is $(L-1)$-hops from $r$. Since the tree has depth $L+1$, either $i$ is a leaf node or all children of $i$ are leaf nodes. If $i$ is a leaf node, then $y_n[i,t]=y_*[i,t]=c[i,t]\in[0,1]$ and the statement holds. If all children of $i$ are leaf nodes, then we have $y_n[j,t]=y_*[j,t]=c[j,t]$ for all $j\in\mathcal{C}_i$. Hence, by Lemma~\ref{lemma:exp3-one-step}, 
\begin{align*}
    &\sum_{t=1}^T E\Big[y_\epsilon[i,t]\Big]=\sum_{t=1}^T E\Big[y_\epsilon[i,t]\Big|\mathcal{Y}_\epsilon[i,t]\Big]\\
    \leq& \min_{j\in\mathcal{C}_i}\sum_{t=1}^Ty_\epsilon[j,t]+(D+\log D)T^{\frac{L}{L+1}}\\
    =&\sum_{t=1}^Ty_*[i,t]+(D+\log D)T^{\frac{L}{L+1}},
\end{align*} 
and the statement holds.

We now assume that the statement holds when $h=g$ and consider a node $i$ that is $(L-(g+1))$-hops from $r$. Either $i$ is a leaf node or all children of $i$ are $(L-g)$-hops from $r$. If $i$ is a leaf node, then the statement clearly holds. If $i$ is not a leaf node, then, by the induction hypothesis, we have $\sum_{t=1}^T E\Big[y_\epsilon[j,t]\Big]\leq \sum_{t=1}^Ty_*[j,t]+((2g-1)D+g\log D)T^{\frac{L}{L+1}}$, for all $j\in\mathcal{C}_i$. 
We can then use Lemma~\ref{lemma:exp3-one-step} to establish the following:
\begin{align*}
    &\sum_{t=1}^T E\Big[y_\epsilon[i,t]\Big]=\sum_{t=1}^T E\Big[E\Big[y_\epsilon[i,t]\Big|\mathcal{Y}_\epsilon[i,t]\Big]\Big]\\
    \leq &E\Big[\min_{j\in\mathcal{C}_i}\sum_{t=1}^Ty_\epsilon[j,t]\Big]+(2D+\log D)T^{\frac{L}{L+1}}\\
    \leq& \min_{j\in\mathcal{C}_i}\sum_{t=1}^Ty_*[j,t]+((2g-1)D+g\log D)T^{\frac{L}{L+1}}\\
    &+(2D+\log D)T^{\frac{L}{L+1}}\\
    =&\sum_{t=1}^Ty_*[i,t]+((2g+1)D+(g+1)\log D)T^{\frac{L}{L+1}},
\end{align*}
and the statement holds. By induction, the statement holds for all $h$.

Since the root node $r$ is 0-hop from itself and $\sum_{t=1}^Ty_*[r,t]=\min_{j\in\mathcal{L}}\sum_{t=1}^Tc_{j,t}$, the theorem holds.
\end{proof}

Finally, we note that the $\epsilon-$EXP3 algorithm requires the knowledge of $T$ to set $\epsilon_i$ and $\eta_i$. When $T$ is not known in advance, we can employ the doubling trick to design an anytime algorithm as shown in Algorithm~\ref{alg:anytime-exp3}. This anytime algorithm is also a no-regret policy:

\begin{algorithm}[h]
   \caption{Anytime $\epsilon$-EXP3}
   \label{alg:anytime-exp3}
\begin{algorithmic}[1]
\FOR{m = 0, 1, 2, \dots}
\STATE Set $\epsilon_i$ and $\eta_i$ according to Theorem~\ref{theorem:eps-EXP3 regret}, but replace $T$ with $2^m$
\STATE Run Algorithm~\ref{alg:exp3} on the $2^m$ rounds $t=2^m, 2^m+1, \dots, 2^{m+1}-1$ 
\ENDFOR
\end{algorithmic}
\end{algorithm}

\begin{theorem}
The regret of Algorithm~\ref{alg:anytime-exp3} is at most $\frac{2^{\frac{2L}{L+1}}}{2^{\frac{L}{L+1}}-1}((2L-1)D+L\log D)T^{\frac{L}{L+1}}$.
\end{theorem}
\begin{proof}
    The proof is very similar to that in \cite[Section 2.3.1]{shalev2012online}, and is hence omitted.
\end{proof}

\section{Regret Lower Bound and the Need for Education} \label{section:lowerbound}

In this section, we establish a regret lower bound of $\Omega(T^{\frac{L-1}{L}})$ for a class of \emph{time-homogeneous oracle policies}. Under this class of policies, each node knows the outcomes of each non-leaf child, $y[i,t]$, \emph{before} selecting a child to forward a job to. Since the outcomes of each non-leaf child is known in advance, there is no need for exploration and each node only faces an education-exploitation dilemma. As we establish a regret lower bound for this class of policies, we also establish the need for education.

\begin{figure}[h]
    \vspace{-10pt}
    \begin{center}
        \includegraphics[width=0.95\linewidth]{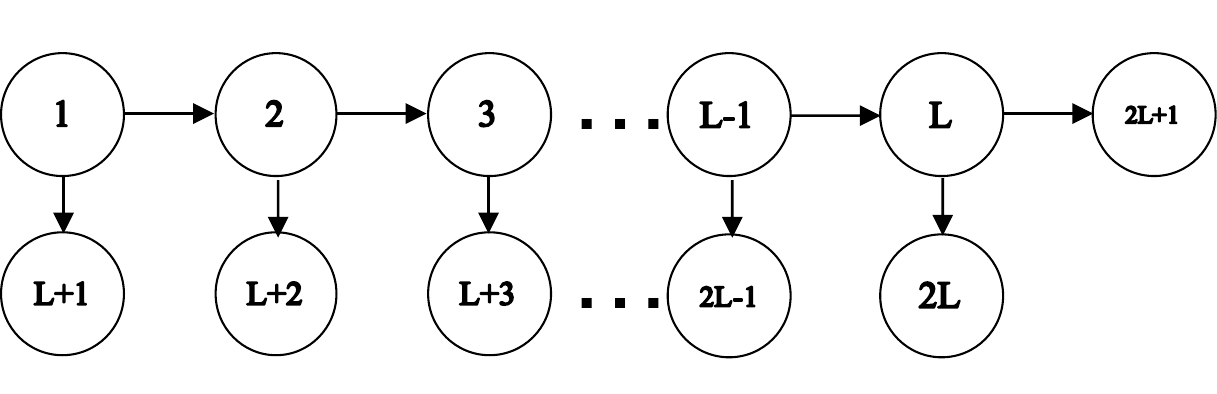}
    \end{center}
    \caption{System illustration for establishing a lower bound}
    \label{fig:lower_bound}
\end{figure}

We consider a system with depth $L+1$ as shown in Fig.~\ref{fig:lower_bound}. There are $L$ non-leaf nodes, numbered as $1, 2, \dots, L$, and $L+1$ leaf nodes, numbered as $L+1, L+2, \dots, 2L+1$. Each non-leaf node $i$ has two children. For each $i\leq L-1$, one child of node $i$ is the leaf node $L+i$ and the other child is the non-leaf node $i+1$. For node $L$, both children, node $2L$ and node $2L+1$, are leaf nodes. When a leaf node $j$ received a job, it generates a cost of 1 with probability $p_j$ and a cost of 0 with probability $1-p_j$. Given a small positive constant $\delta<1/2^{L}$, we set one of $p_{2L}$ and $p_{2L+1}$ to be $(1-2^L\delta)/2$ and the other to be $(1+2^L\delta)/2$, and then set $p_{j}=(1-(2^L-2^{j-L-1})\delta)/2$ for all the other leaf nodes $j=L+1, L+2, \dots, 2L-1$. Hence, we have $(1-2^L\delta)/2<p_{L+1}<p_{L+2}<\dots<p_{2L-1}<(1-2^L\delta)/2$ and $\min_{j\in\mathcal{L}}\sum_{t=1}^TE\big[c[j,t]\big]=(1-2^L\delta)T/2$. The regret of the system is 
\begin{equation}
    \sum_{t=1}^TE\big[y[1,t]\big]-(1-2^L\delta)T/2=\sum_{t=1}^T\Big(E\big[y[1,t]\big]-(1-2^L\delta)/2\Big).
\end{equation}

We now discuss the policies employed by each non-leaf node. Since both children of node $L$ are leaf nodes, node $L$ does not need to consider education. We consider that node $L$ can run an arbitrary online learning algorithm with bandit feedback. For all other non-leaf nodes $i=1, 2, \dots, L-1$, we assume that they employ a time-homogeneous oracle policy defined as follows:

\begin{definition}
    Let $i_1$ and $i_2$ be the two children of node $i$, then a time-homogeneous oracle policy is one that chooses a child to forward a job to at time $t$ with the following assumptions:
    \begin{itemize}
        \item A1: Node $i$ can obtain the expected cost of each child, $E\big[y[i_1,t]\big]$ and $E\big[y[i_2,t]\big]$, before making the forwarding decision.
        \item A2: Node $i$ makes its forwarding decision solely based on $E\big[y[i_1,t]\big]-E\big[y[i_2,t]\big]$. Specifically, let $\zeta:=\Big|E\big[y[i_1,t]\big]-E\big[y[i_2,t]\big]\Big|$, then node $i$ will forward the job to the child with the higher expected cost with probability $P_i(\zeta)$, and to the other child with probability $1-P_i(\zeta)$, where $P_i(\cdot)$ is an arbitrary decreasing function chosen by node $i$.
    \end{itemize}
\end{definition}

We note that A1 provides a node with much more information than is possible in multi-stage systems with bandit feedback, where a node can only obtain the cost of a child \emph{if} it forwards a job to the child, and only \emph{after} it makes the forwarding decision. Thus, intuitively, the regret of policies with A1 serves as a natural lower bound for the regret of policies with end-to-end bandit feedback. The purpose of A2 is to highlight that a node $i$ only knows the expected costs, but not the internal variables of its children.

We also note that policies with A1 do not need to explore, since it knows the expected costs of all children in advance. Hence, policies with A1 only face an education-exploitation dilemma. The only reason that a policy may select a child with a higher expected cost, by choosing $P_i(\eta)>0$, is to educate its children.

We first establish a bound for the expected cost of node $L$, whose children are both leaf nodes. Let $N_L(t)$ be the number of times that node $L$ has received a job from its parent at time $t$. Since node $L$ can only learn the costs of its children when it receives a job, node $L$ cannot determine which of its two children has the smaller $p_j$ when $N_L(t)$ is small. The following lemma formalizes this intuition. 

\begin{lemma}   \label{lemma:node_L_regret}
There exists a positive integer $N_\delta$ such that, for all $t$ with $N_L(t)<N_\delta$, $E\big[y[L,t]\big]>(1-(2^L-2^{L-1})\delta)/2$.
\end{lemma}
\begin{proof}
This is a direct result of Lemma 3.6 in \cite{bubeck2012regret}.    
\end{proof}

We now establish a regret lower bound for the system in Fig.~\ref{fig:lower_bound}.

\begin{theorem}
    For the system in Fig.~\ref{fig:lower_bound}, the regret is $\Omega(T^{\frac{L-1}{L}})$ for any bandit learning policy employed by node $L$ and any time-homogeneous oracle policies employed by nodes $1, 2,\dots, L-1$. 
\end{theorem}
\begin{proof}
Let $T_\delta$ be the time at which $N_L(t)=N_\delta$. By Lemma~\ref{lemma:node_L_regret}, $E\big[y[L,t]\big]>(1-(2^L-2^{L-1})\delta)/2$ for any $t<T_\delta$.

We first study the system behavior before time $T_\delta$. Consider the forwarding decision of node $L-1$ at any time $t<T_\delta$. Node $L-1$ has two children. One is the leaf node $2L-1$ with $E\big[y[2L-1,t]\big]=p_{2L-1}=(1-(2^L-2^{L-2})\delta)/2$. The other is the non-leaf node $L$ with $E\big[y[L,t]\big]>(1-(2^L-2^{L-1})\delta)/2=p_{2L-1}+2^{L-3}\delta$. By A2, the probability that node $L-1$ selects node $L$ is at most $q_{L-1}:=P_{L-1}(2^{L-3}\delta)$. We also have $E\big[y[L-1,t]\big]\geq p_{2L-1}=(1-(2^L-2^{L-2})\delta)/2$.

We further analyze the forwarding decision of node $i<L-1$ at any time $t<T_\delta$. Using a simple induction argument, it can be shown that the probability that node $i$ selects node $i+1$ is at most $q_{i}:=P_i(2^{i-2}\delta)$, and $E\big[y[i,t]\big]\geq p_{L+i}=(1-(2^L-2^{i-1})\delta)/2$. Therefore, at any time $t<T_\delta$, we have
\begin{equation}\label{eq:lowerboundproof1}
    E\big[y[1,t]\big] - (1-2^L\delta)/2\geq \delta/2.
\end{equation}

Moreover, since node $L$ can only receive a job if, for each $i\leq L-1$, node $i$ selects node $i+1$, which happens with probability at most $q_i$, we have
\begin{equation}\label{eq:lowerboundproof2}
    E\big[T_\delta]\geq \frac{N_\delta}{\prod_{i=1}^{L-1}q_i}.
\end{equation}

Next, we analyze the system behavior after time $T_\delta$. For any time $t>T_\delta$, $E\big[y[L,t]\big]\geq \min\{p_{2L}, p_{2L+1}\}=(1-2^L\delta)/2$. Consider the forwarding decision of node $L-1$. Since $E\big[y[2L-1,t]\big]=p_{2L-1}=(1-(2^L-2^{L-2})\delta)/2\leq E\big[y[L,t]\big]+2^{L-3}\delta$, the probability that node $L-1$ selects node $L$ is at most $1-P_{L}(2^{L-3}\delta)=1-q_{L-1}$. Using a simple induction argument, we can further show that the probability that node $i$ selects node $i+1$ is at most $1-q_i$, for all $i\leq L-1$. Hence, the probability that node $L$ receives a job is at most $\prod_{i=1}^{L-1}(1-q_i)$. If node $L$ does not receive a job, which happens with probability at least $1-\prod_{i=1}^{L-1}(1-q_i)$, then the expected cost is at least $\min_{j\in\{L+1, L+2, \dots, 2L-1\}}p_j=(1-(2^L-1)\delta)/2$. We then have, at any time $t>T_\delta$,
\begin{equation}\label{eq:lowerboundproof3}
    E\big[y[1,t]\big] - (1-2^L\delta)/2\geq \big(1-\prod_{i=1}^{L-1}(1-q_i)\big)\delta/2.
\end{equation}

Combining Eq. (\ref{eq:lowerboundproof1}), (\ref{eq:lowerboundproof2}), and (\ref{eq:lowerboundproof3}) and we have the following regret bound
\begin{align}\label{eq:lowerboundproof4}
    &\sum_{t=1}^T\Big(E\big[y[1,t]\big]-(1-2^L\delta)/2\Big)\nonumber\\
    \geq& E\big[\sum_{t=1}^{T_\delta}\delta/2+\sum_{t=T_\delta+1}^T\big(1-\prod_{i=1}^{L-1}(1-q_i)\big)\delta/2\big]\nonumber\\
    \geq&\frac{\delta}{2}[\frac{N_\delta}{\prod_{i=1}^{L-1}q_i}+(T-\frac{N_\delta}{\prod_{i=1}^{L-1}q_i})\big(1-\prod_{i=1}^{L-1}(1-q_i)\big)].
\end{align}

It is then straightforward to show that $\frac{\delta}{2}[\frac{N_\delta}{\prod_{i=1}^{L-1}q_i}+(T-\frac{N_\delta}{\prod_{i=1}^{L-1}q_i})\big(1-\prod_{i=1}^{L-1}(1-q_i)\big)]=\Omega(T^{\frac{L-1}{L}})$. Moreover, setting $q_i=\Theta(T^{-\frac{1}{L}})$, for all $i\leq L-1$, makes $\frac{\delta}{2}[\frac{N_\delta}{\prod_{i=1}^{L-1}q_i}+(T-\frac{N_\delta}{\prod_{i=1}^{L-1}q_i})\big(1-\prod_{i=1}^{L-1}(1-q_i)\big)]=\Theta(T^{\frac{L-1}{L}})$.
\end{proof}

{Before closing the section, we note that the lower-bound analysis in this section is limited to time-homogeneous policies. We make this assumption to explicitly prevent a parent node from using history to imply internal variables of its children. Extending our analysis to time-varying policies will be interesting future work.}
\section{Simulation Results}
\label{section:simulation}

\begin{figure*}[t]
    \begin{center}
    \subfigure[$D=2, L=2, p_{min}=0.2$]
    {
       \includegraphics[width=0.3\linewidth]{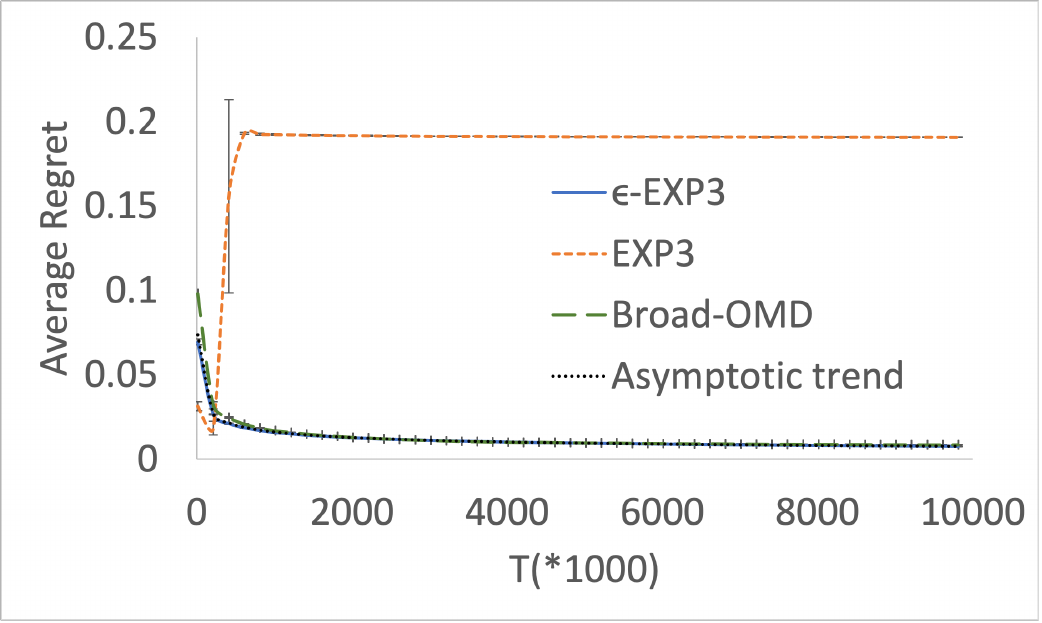}
       \label{fig:D2L2}
    }
    \subfigure[$D=2,L=3, p_{min}=0.4$]
    {
       \includegraphics[width=0.3\linewidth]{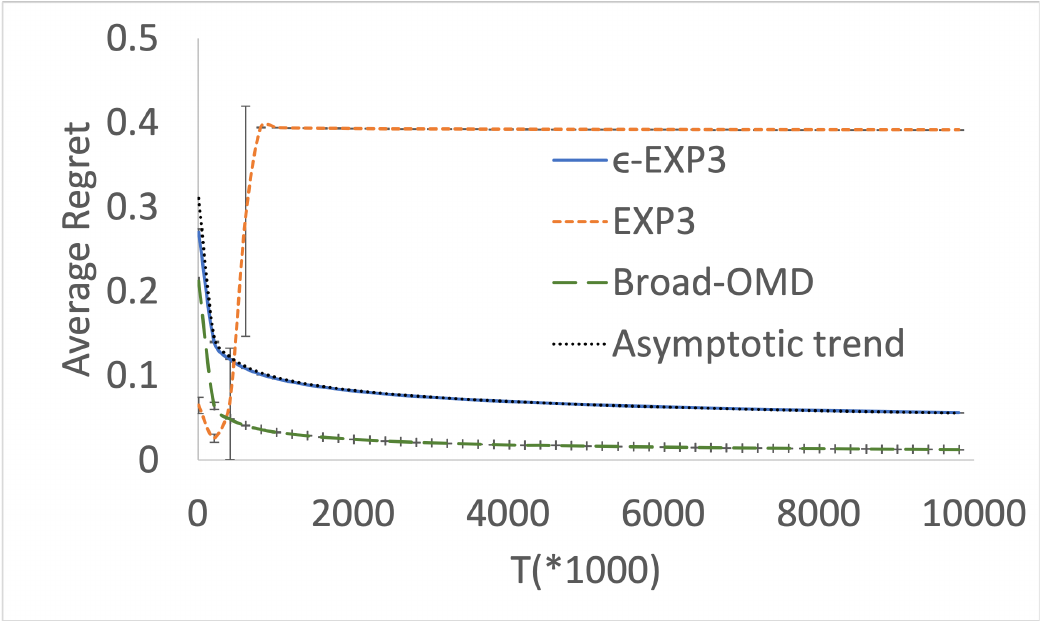}
       \label{fig:D2L3}
    }
    \subfigure[$D=2,L=4, p_{min}=0.6$]
    {
       \includegraphics[width=0.3\linewidth]{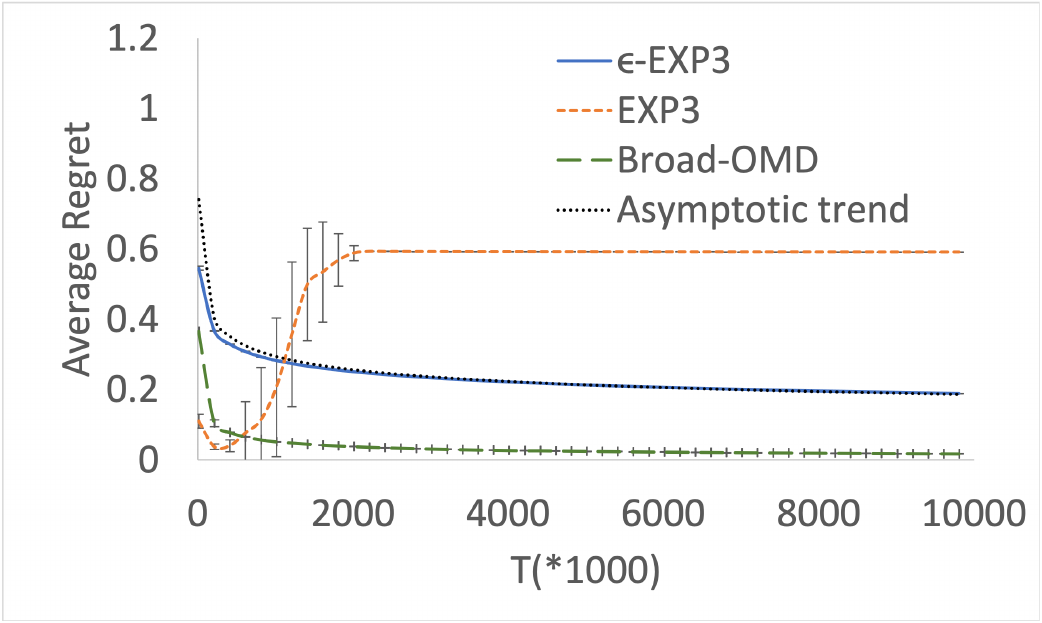}
       \label{fig:D2L4}
    }
    \subfigure[$D=4, L=2, p_{min}=0.2$]
    {
       \includegraphics[width=0.3\linewidth]{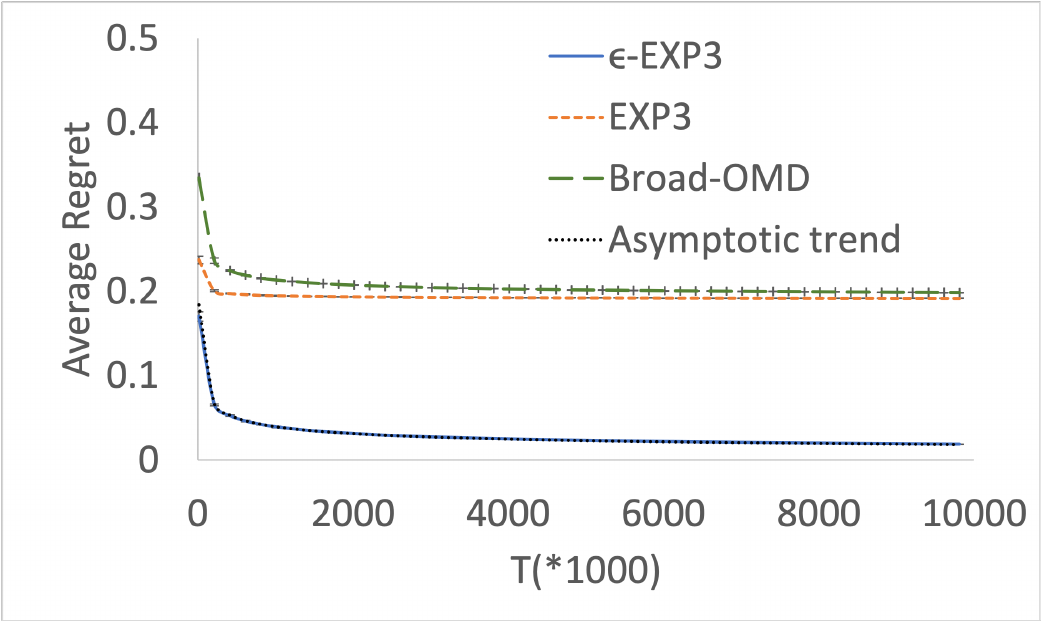}
       \label{fig:D4L2}
    }
    \subfigure[$D=4,L=3, p_{min}=0.4$]
    {
       \includegraphics[width=0.3\linewidth]{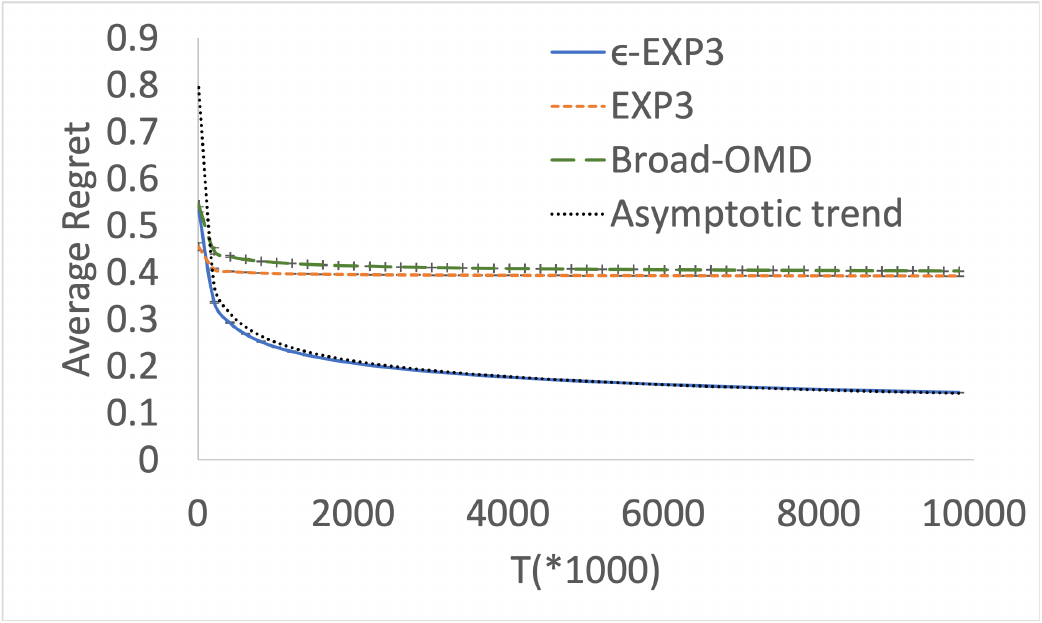}
       \label{fig:D4L3}
    }
    \subfigure[$D=4,L=4, p_{min}=0.6$]
    {
       \includegraphics[width=0.3\linewidth]{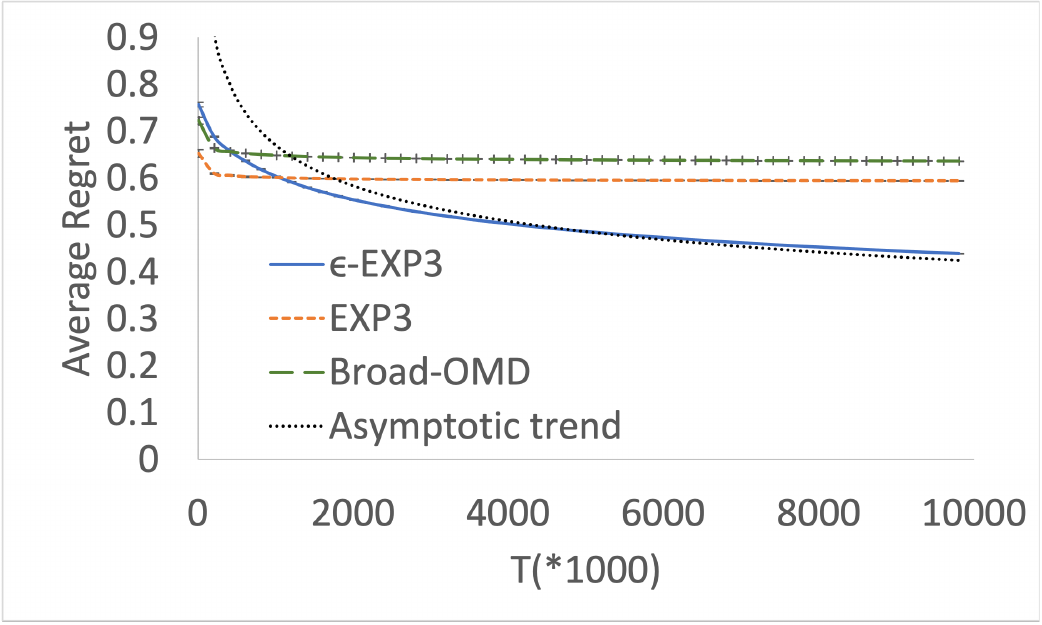}
       \label{fig:D4L4}
    }
    \end{center}
    \caption{Time-average regrets under various system parameters}
    \label{fig:D*L*}
\end{figure*}

We present our simulation results in this section. We simulate two different scenarios. The first scenario is based on trees whose leaf nodes generate Bernoulli costs. While this scenario is artificially constructed and may not correspond to real-world applications, its simulation results provide important insights on how online algorithms behave in distributed multi-stage systems. The second scenario is based on mobile edge computing. We compare our $\epsilon-$EXP3, with parameters from Lemma~\ref{lemma:exp3-one-step}, against the standard EXP3 algorithm, where each node runs the EXP3 algorithm independently from each other, and the Broad-OMD algorithm \cite{wei2018more}. Both EXP3 and Broad-OMD are no-regret policies for the special case when $L=1$.

\subsection{Trees with Bernoulli Costs} \label{section:simulation:tree}
We consider systems that can be represented as trees with depth $L+1$. Each non-leaf node has $D$ children. Each leaf node $j$ is associated with a parameter $p_j\in[0,1]$. Whenever a leaf node $j$ receives a job, its cost $c[j,t]$ is 1 with probability $p_j$ and 0 with probability $1-p_j$. The system is run over $T$ rounds. Initially, the values of $p_j$ is chosen so that $\max_j p_j = 1$ and $\min_j p_j = p_{min}$. At round $t=T/100$, the leaf with $p_j=1$ has its value of $p_j$ changed into 0. Fig. \ref{fig:simu_example} illustrates an example. For a given set of parameters $D, L, T$, and $[p_j]$, we simulate the system for 20 independent runs and calculate the time-average regret $\Big(\sum_{t=1}^Ty[r,t]-\min_{j\in\mathcal{L}}\sum_{t=1}^Tc[j,t]\Big)/T$ under all evaluated policies. 
\begin{figure}[h]
    \vspace{-10pt}
    \begin{center}
        \includegraphics[width=0.8\linewidth]{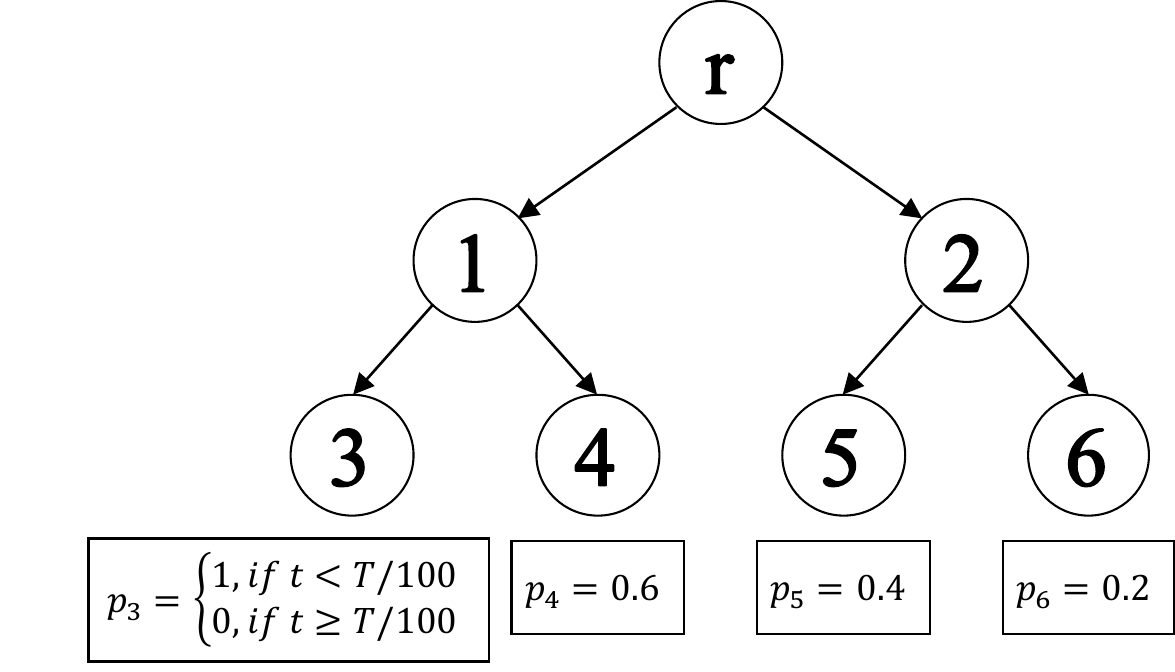}    
    \end{center}
    \caption{A system with $D=2, L=2$ and $p_{min}=0.2$}
    \label{fig:simu_example}
\end{figure}

Simulation results are shown in Fig.~\ref{fig:D*L*}, with the error bars indicating standard deviations. It can be observed that the time-average regret of $\epsilon-$EXP3 approaches 0 over time in all cases. We note that the convergence rate of $\epsilon-$EXP3 becomes much slower as $L$ becomes larger. This is consistent with Theorem~\ref{theorem:eps-EXP3 regret}, which shows that the time-average regret scales as $O(1/\sqrt[L+1]{T})$. To verify that the time-average regret of $\epsilon-$EXP3 scales as $O(1/\sqrt[L+1]{T})$, we also plot the \emph{asymptotic trend} in Fig.~\ref{fig:D*L*}. The value of the asymptotic trend for a particular $T$ is calculated as $R_{D,L}/\sqrt[L+1]{T}$, where $R_{D,L}$ is chosen so that the value of the asymptotic trend and the time-average regret of $\epsilon-$EXP3 are the same when $T=5\times10^6$, that is, at the mid-point of the x-axis in the figures. It can be observed that $\epsilon-$EXP3 is close to the asymptotic trend. This demonstrates that the time-average regret $\epsilon-$EXP3 indeed scales as $O(1/\sqrt[L+1]{T})$. 

On the other hand, it can also be observed that the time-average regrets of both EXP3 and Broad-OMD converge to $p_{min}$ in all settings in Fig.~\ref{fig:D*L*}. This result shows that neither of them is a no-regret policy in multi-stage systems. To understand why the standard EXP3 algorithm is not a no-regret policy, consider the system illustrated in Fig.~\ref{fig:simu_example}. Before round $t=\frac{T}{100}$, the optimal strategy for node 1 is to forward the job to node 4 with $p_4=0.6$. The optimal strategy for the root is to forward the job to node 2, who then forwards the job to node 6 with $p_6=0.2$. Hence, at round $t=\frac{T}{100}$ and under the EXP3 algorithm, the root will choose node 2 with a high probability and node 1 will choose node 4 with a high probability. Now, consider the first time after round $\frac{T}{100}$ when the root forwards a job to node 1. Since node 1 is unaware that $p_3$ has become 0, it chooses node 4 with a high probability and will likely incur a high cost. This high cost will cause the root to exponentially reduce the probability of choosing node 1 in the future, making it even harder for node 1 to explore and learn the fact that $p_3$ has become 0. This is why the EXP3 algorithm suffers from a time-average regret of roughly $p_6=0.2$. In contrast, our $\epsilon-$EXP3 algorithm ensures that the root always chooses node 1 with at least a constant probability in each round. This persistent education enables node 1 to eventually discover that $p_3$ has become 0. 

To demonstrate the behavior discussed in the above paragraph, we conduct a simulation to show the transient behaviors of the two algorithms. Specifically, we test the system shown in Fig.~\ref{fig:simu_example} with $T=5\times10^6$. The value of $p_3$ is initially 1, and becomes 0 at round $5\times 10^4$. For each algorithm, we record the probability that the root $r$ would choose node 1 and the probability that node 1 would choose node 3. Simulation results are shown in Fig.~\ref{fig:transient}, where each data point represent the average of the previous 1000 rounds. Under the EXP3 algorithm, the probability that the root would choose node 1 at round $5\times 10^4$ is less than $0.05\%$. Since node 1 rarely receives any jobs, it cannot improve its performance, which, in turn, makes the root even less likely to choose node 1. At round $5\times 10^5$, probability that the root would choose node 1 has become less than $0.02\%$. In contrast, the $\epsilon-$EXP3 algorithm offers persistent education to node 1. Hence, after round $5\times 10^4$, node 1 quickly finds that $p_3$ has improved and increases its probability of choosing node 3. As a result, the root also starts increasing its probability of choosing node 1 after round $10^5$. 

\begin{figure}
    \begin{center}
    \subfigure[The behavior of $\epsilon-$EXP3]
    {
       \includegraphics[width=0.45\linewidth]{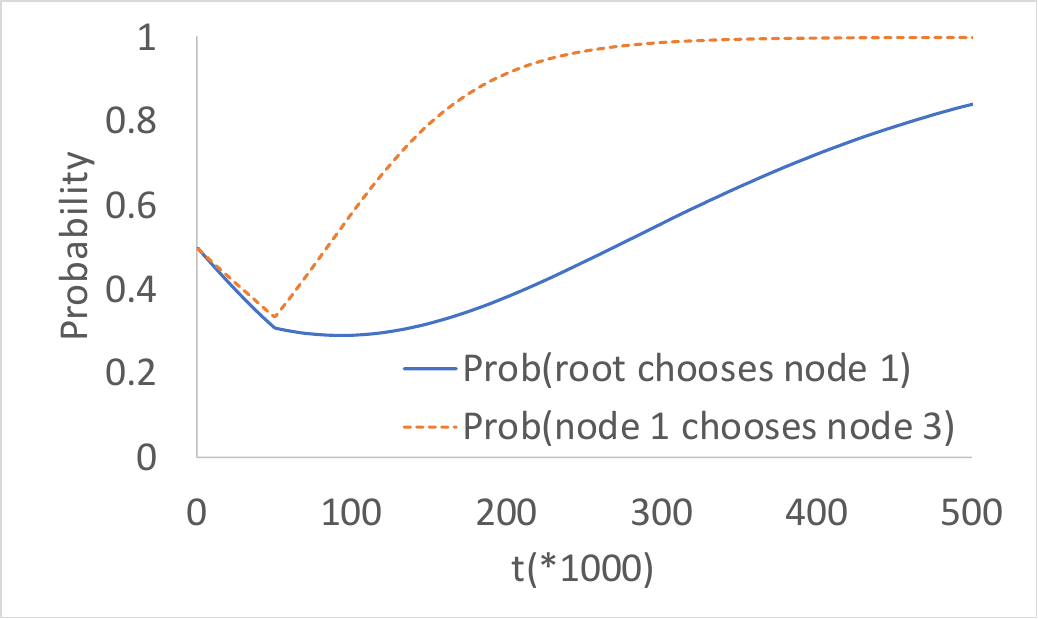}
       \label{fig:transient_epsilon}
    }
    \subfigure[The behavior of EXP3]
    {
       \includegraphics[width=0.45\linewidth]{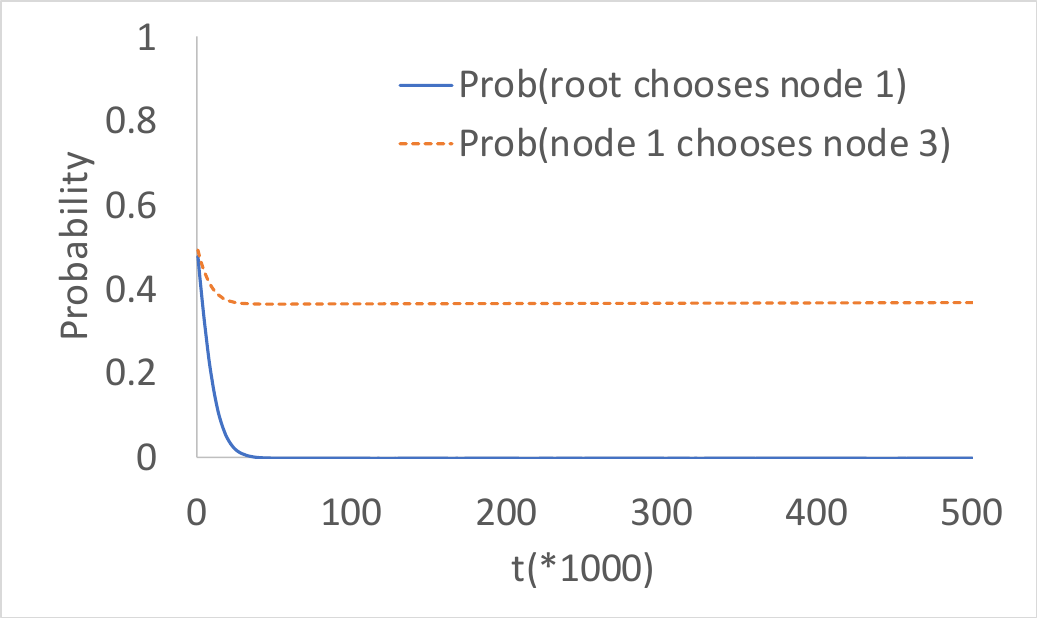}
       \label{fig:transient_exp3}
    }
    \end{center}
    \caption{Transient behaviors of the system in Fig.~\ref{fig:simu_example} with $T=5\times10^6$.}
    \label{fig:transient}
\end{figure}

\subsection{Mobile Edge Computing}

We consider a mobile edge computing system. In this system, there is a mobile robot that generates video analytic jobs for real-time processing. The robot is connected to $D$ edge servers with different communication media. To process a job, each edge server has $D$ different neural networks to choose from. Different neural networks have different precision and different processing time. 
There is also a communication latency of each link. The delay of transmitting over a link is an exponential function with mean $\frac{1}{\lambda}$. Some links have a constant $\lambda$ while other links have a $\lambda$ that increases over time. This models the time-varying congestion on these links. Fig.~\ref{fig:mec_system} illustrates the system when $D=2$.

\begin{figure*}[h]
    \centering
    \subfigure[System illustration when $D=2$]
    {
    \includegraphics[width=0.35\linewidth]{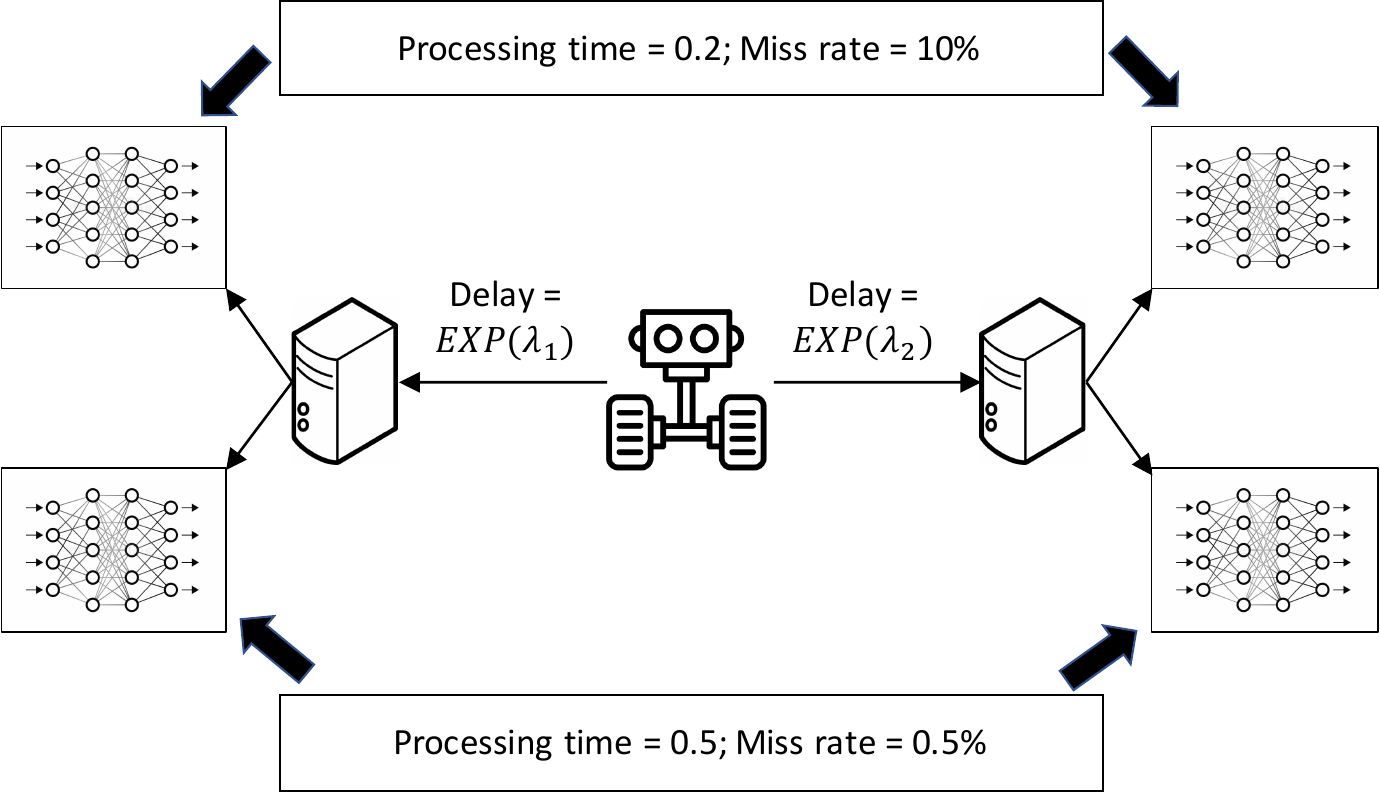}
    \label{fig:mec_system}
    }
    \subfigure[Results for $D=2$]
    {
    \includegraphics[width=0.28\linewidth]{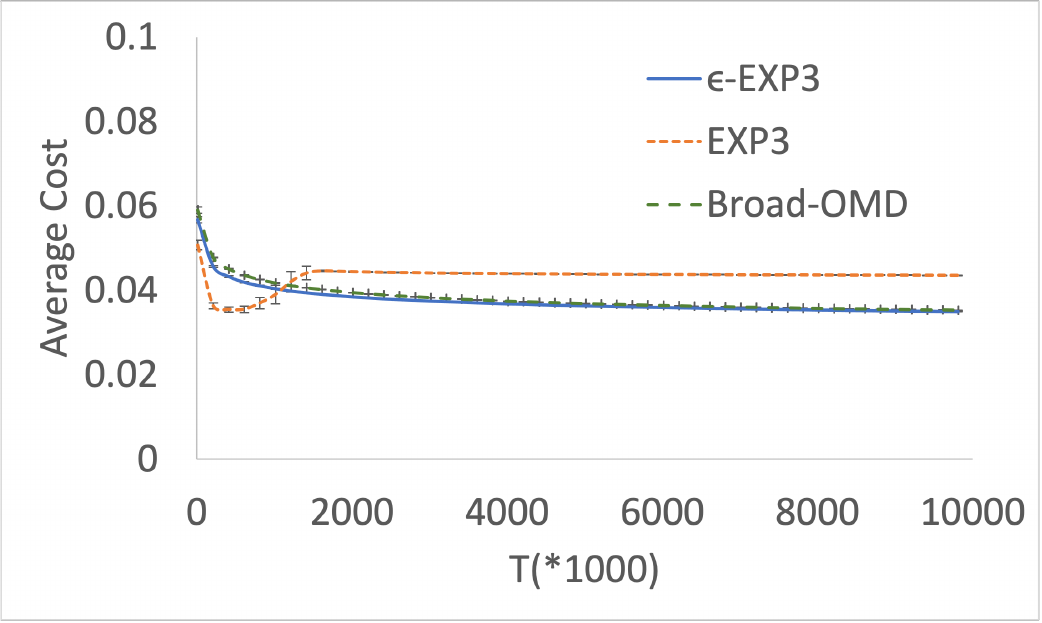}
    \label{fig:mec_result}
    }
    \subfigure[Results for $D=3$]
    {
    \includegraphics[width=0.28\linewidth]{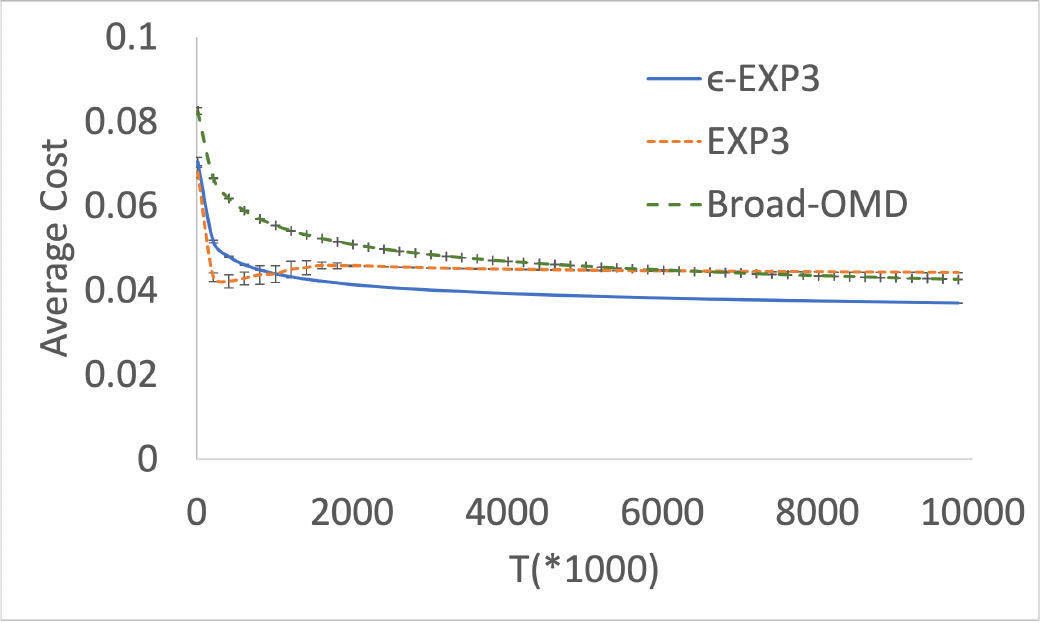}
    \label{fig:mec_result}
    }
    \caption{Setting and result of a mobile edge computing system}
    \label{fig:mec}
\end{figure*}

The robot requires a strict deadline of one time unit for each job. If the end-to-end latency, that is, the sum of communication latency and processing time, exceeds one time unit, then a deadline violation occurs and the cost is one. If the end-to-end latency is less than one time unit, then the cost is the miss rate of the employed neural network.

We have conducted 20 independent runs for each $T$. Simulation results are shown in Fig.~\ref{fig:mec}. It can be observed that the $\epsilon-$EXP3 algorithm significantly outperforms the EXP3 algorithm when $T$ is sufficiently large. The Broad-OMD algorithm has similar performance as $\epsilon-$EXP3 when $D=2$, but is much worse than $\epsilon-$EXP3 when $D=3$.

\subsection{Multi-hop Networks}
We consider multi-hop networks as illustrated in Fig.~\ref{fig:net_system}. In this system, the source (node $S$) is sending packets to the destination (node $D$) through a number of inter-connected relay nodes. Upon receiving a packet, a node needs to decide which node to forward the packet to. The delay of transmitting over a link is an exponential function with mean $\frac{1}{\lambda}$. Some links have a constant $\lambda$ while other links have a $\lambda$ that increases over time. We consider that the source requires a strict end-to-end deadline guarantee of one unit time. If the end-to-end delay of a packet is more than one unit time, then a deadline violation occurs.

\begin{figure*}
    \centering
    \subfigure[System illustration]
    {
    \includegraphics[width=0.3\linewidth]{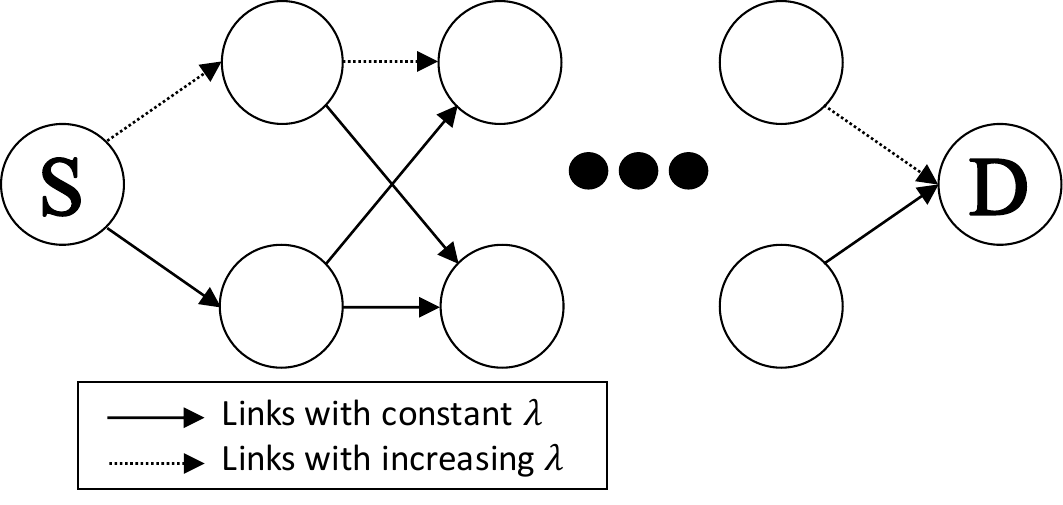}
    \label{fig:net_system}
    }
    \subfigure[Results for $L=2$]
    {
    \includegraphics[width=0.3\linewidth]{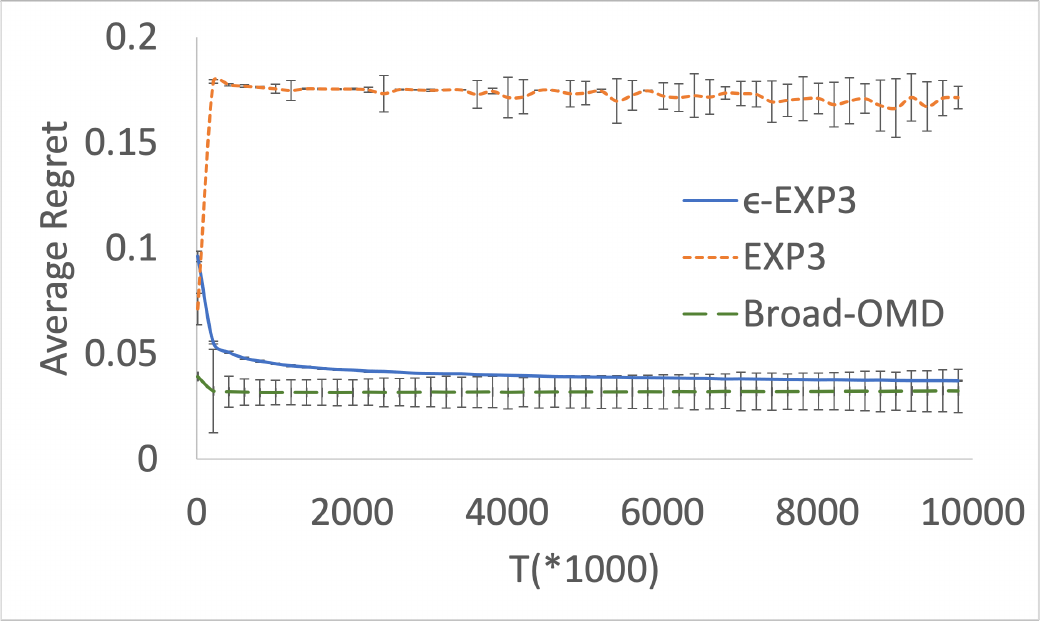}
    \label{fig:net_L2}
    }
    \subfigure[Results for $L=3$]
    {
    \includegraphics[width=0.3\linewidth]{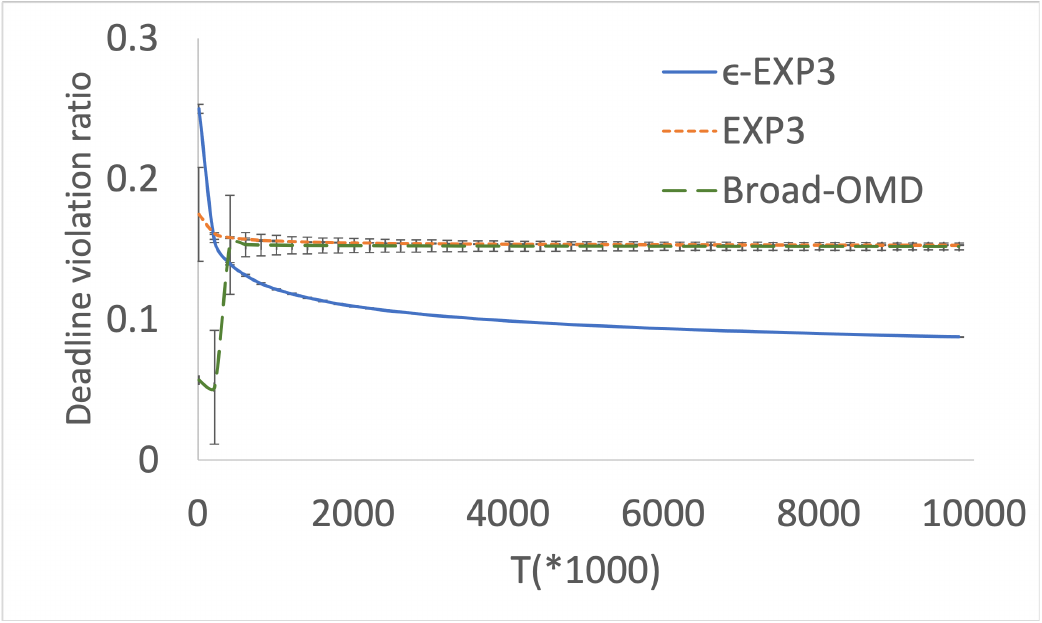}
    \label{fig:net_L3}
    }
    \caption{Setting and result of multi-hop networks}
    \label{fig:net}
\end{figure*}

Let $L$ be the number of relay nodes that a packet needs to visit before reaching the destination. We have tested this system for different values of $L$. Simulation results are shown in Fig.~\ref{fig:net}. It can be observed that the $\epsilon-$EXP3 algorithm is either optimal or near-optimal in all settings.

\section{Conclusion} \label{section:conclusion}

In this paper, we study multi-stage systems with end-to-end bandit feedback. The fundamental challenge of learning the optimal policy of agents in each stage is a newly introduced exploration-exploitation-education trilemma. We propose a simple distribute policy, the $\epsilon-$EXP3 algorithm, that explicitly addresses this trilemma. Moreover, we theoretically prove that the $\epsilon-$EXP3 algorithm is a no-regret policy. Simulation results show that the $\epsilon-$EXP3 algorithm significantly outperforms existing policies.

\section{Acknowledgement}
This material is based upon work supported in part by NSF under Award Numbers ECCS-2127721 and CCF-2332800 and in part by the U.S. Army Research Laboratory and the U.S. Army Research Office under Grant Number W911NF-22-1-0151.

\bibliographystyle{ACM-Reference-Format}
\bibliography{reference}
\newpage
\appendix
\section{Proof of Theorem \ref{theorem:regret of normalized-EG}}\label{appendix:thm1}
\begin{proof}
We will prove the theorem by establishing the following statement: If a node $i$ is $(L-h)$-hops from the root node $r$, then $\sum_{t=1}^T E\Big[y_n[i,t]\Big]\leq \sum_{t=1}^Ty_*[i,t]+2h\sqrt{T\log D}$.

We prove the statement by induction. First, consider the case $h=1$, that is, the node $i$ is $(L-1)$-hops from $r$. Since the tree has depth $L+1$, either $i$ is a leaf node or all children of $i$ are leaf nodes. If $i$ is a leaf node, then $y_n[i,t]=y_*[i,t]=c[i,t]\in[0,1]$ and the statement holds. If all children of $i$ are leaf nodes, then we have $y_n[j,t]=y_*[j,t]=c[j,t]$ for all $j\in\mathcal{C}_i$. Hence, by Lemma~\ref{lemma:omd}, 
\begin{align*}
    &\sum_{t=1}^T E\Big[y_n[i,t]\Big]=\sum_{t=1}^T E\Big[y_n[i,t]\Big|\mathcal{Y}_n[i,t]\Big]\\
    \leq& \min_{j\in\mathcal{C}_i}\sum_{t=1}^Ty_n[j,t]+2\sqrt{T\log |\mathcal{C}_i|}\\
    \leq& \min_{j\in\mathcal{C}_i}\sum_{t=1}^Ty_*[j,t]+2\sqrt{T\log D}
    =\sum_{t=1}^Ty_*[i,t]+2\sqrt{T\log D},
\end{align*} 
and the statement holds.

We now assume that the statement holds when $h=g$ and consider a node $i$ that is $(L-(g+1))$-hops from $r$. Either $i$ is a leaf node or all children of $i$ are $(L-g)$-hops from $r$. If $i$ is a leaf node, then the statement clearly holds. If $i$ is not a leaf node, then, by the induction hypothesis, we have $\sum_{t=1}^T E\Big[y_n[j,t]\Big]\leq \sum_{t=1}^Ty_*[j,t]+2g\sqrt{T\log D}$, for all $j\in\mathcal{C}_i$. Since $y_{n}[j,t]\in[0,1]$, we can use Lemma~\ref{lemma:omd} to establish the following:
\begin{align*}
    &\sum_{t=1}^T E\Big[y_n[i,t]\Big]=\sum_{t=1}^T E\Big[E\Big[y_n[i,t]\Big|\mathcal{Y}_n[i,t]\Big]\Big]\\
    \leq &E\Big[\min_{j\in\mathcal{C}_i}\sum_{t=1}^Ty_n[j,t]\Big]+2\sqrt{T\log |\mathcal{C}_i|}\\
    \leq &\min_{j\in\mathcal{C}_i}\sum_{t=1}^TE\Big[y_n[j,t]\Big]+2\sqrt{T\log D}\\
    \leq& \min_{j\in\mathcal{C}_i}\sum_{t=1}^Ty_*[j,t]+2(g+1)\sqrt{T\log D}\\
    =&\sum_{t=1}^Ty_*[i,t]+2(g+1)\sqrt{T\log D},
\end{align*}
and the statement holds. By induction, the statement holds for all $h$.

Since the root node $r$ is 0-hop from itself, we have $\sum_{t=1}^T E\Big[y_n[r,t]\Big]\leq \sum_{t=1}^Ty_*[r,t]+2L\sqrt{T\log D}=\min_{j\in\mathcal{L}}\sum_{t=1}^Tc_{j,t}+2L\sqrt{T\log D}.$
\end{proof}

\section{Proof of Lemma~\ref{lemma:z[j,t]}}\label{appendix:lemma2}
\begin{proof}
Under the $\epsilon-$EXP3 algorithm, $z[j,t]\neq0$ only when node $i$ receives a job, which happens with probability $v[i,t]$, and $i$ chooses $f[i,t]=j$, whose probability depends on $m[i,t]$. Hence,
\begin{align*}
    &E\Big[z[j,t]\Big|\mathcal{Y}_\epsilon[i,t],\mathcal{Z}[i,t-1]\Big]\\
    =&\epsilon_i E\Big[z[j,t]\Big|m[i,t]=U,\mathcal{Y}_\epsilon[i,t],\mathcal{Z}[i,t-1]\Big]\\
    &+(1-\epsilon_i)E\Big[z[j,t]\Big|m[i,t]=E,\mathcal{Y}_\epsilon[i,t],\mathcal{Z}[i,t-1]\Big]\\
    =&\epsilon_i v[i,t]\frac{1}{|\mathcal{C}_i|}\frac{y_\epsilon[j,t]|\mathcal{C}_i|}{v[i,t]}\\
    &+(1-\epsilon_i)v[i,t]\frac{e^{\eta_i\theta[i,j,t]}}{\sum_{k\in\mathcal{C}_i}e^{\eta_i\theta[i,k,t]}}\frac{y_\epsilon[j,t]\sum_{k\in\mathcal{C}_i}e^{\eta_i\theta[i,k,t]}}{v[i,t]e^{\eta_i\theta[i,j,t]}}\\
    =&y_\epsilon[j,t],
\end{align*}
and
\begin{align*}
    &E\Big[z[j,t]^2\Big|\mathcal{Y}_\epsilon[i,t],\mathcal{Z}[i,t-1]\Big]\\
    =&\epsilon_i E\Big[z[j,t]^2\Big|m[i,t]=U,\mathcal{Y}_\epsilon[i,t],\mathcal{Z}[i,t-1]\Big]\\
    &+(1-\epsilon_i)E\Big[z[j,t]^2\Big|m[i,t]=E,\mathcal{Y}_\epsilon[i,t],\mathcal{Z}[i,t-1]\Big]\\
    =&\epsilon_i v[i,t]\frac{1}{|\mathcal{C}_i|}\frac{y_\epsilon[j,t]^2|\mathcal{C}_i|^2}{v[i,t]^2}\\
    &+(1-\epsilon_i)v[i,t]\frac{e^{\eta_i\theta[i,j,t]}}{\sum_{k\in\mathcal{C}_i}e^{\eta_i\theta[i,k,t]}}\frac{y_\epsilon[j,t]^2(\sum_{k\in\mathcal{C}_i}e^{\eta_i\theta[i,k,t]})^2}{v[i,t]^2e^{2\eta_i\theta[i,j,t]}}\\
    =&\Big(\epsilon_i|\mathcal{C}_i|+(1-\epsilon_i) \frac{\sum_{k\in\mathcal{C}_i}e^{\eta_i\theta[i,k,t]}}{e^{\eta_i\theta[i,j,t]}}\Big)\frac{y_\epsilon[j,t]^2}{v[i,t]}.
\end{align*}
\end{proof}

\section{Proof of Lemma~\ref{lemma:y_n vs y_e}}\label{appendix:lemma3}
\begin{proof}
    Since both $\epsilon-$EXP3 and normalized-EG update $\theta[i,j,t]$ by $\theta[i,j,t+1]=\theta[i,j,t]-z[j,t]$, they have the same values of $\theta[i,j,t]$ on every sample path.

By the design of the $\epsilon-$EXP3 algorithm, we have
\begin{align*}
&E\Big[y_\epsilon[i,t]\Big|m[i,t]=E,\mathcal{Y}_\epsilon[i,t],\mathcal{Z}[i,t-1]\Big]\\
=&\sum_{j\in\mathcal{C}_i} y_\epsilon[j,t]\frac{e^{\eta_i\theta[i,j,t]}}{\sum_{k\in\mathcal{C}_i}e^{\eta_i\theta[i,k,t]}}.
\end{align*}

Under the normalized-EG algorithm, we have $y_n[i,t]=z[j,t]$ with probability $\frac{e^{\eta_i\theta[i,j,t]}}{\sum_{k\in\mathcal{C}_i}e^{\eta_i\theta[i,k,t]}}$. Hence,
\begin{align*}
    &E\Big[E\Big[y_n[i,t]\Big|\mathcal{Y}_n[i,t]=\mathcal{Z}[i,t]\Big]\Big]\\
    =&\sum_{j\in\mathcal{C}_i}\frac{e^{\eta_i\theta[i,j,t]}}{\sum_{k\in\mathcal{C}_i}e^{\eta_i\theta[i,k,t]}}E\Big[z[j,t]\Big|\mathcal{Y}_\epsilon[i,t],\mathcal{Z}[i,t-1]\Big]\\
    =&\sum_{j\in\mathcal{C}_i} y_\epsilon[j,t]\frac{e^{\eta_i\theta[i,j,t]}}{\sum_{k\in\mathcal{C}_i}e^{\eta_i\theta[i,k,t]}}. \hspace{20pt} (\because\mbox{Eq. (\ref{equation:z[j,t]})})
\end{align*}
This completes the proof.
\end{proof}

\section{Proof of Lemma~\ref{lemma:exp3-one-step}}\label{appendix:lemma4}
\begin{proof}
    The normalized-EG algorithm sets $x[i,j,t]=\frac{e^{\eta_i\theta[i,j,t]}}{\sum_{k\in\mathcal{C}_i}e^{\eta_i\theta[i,k,t]}}$. By Lemma \ref{lemma:omd} and the fact that $y_\epsilon[j,t]\in[0,1],\forall j,t$, we have
\begin{align*}
    &\sum_{t=1}^T E\Big[E\Big[y_n[i,t]\Big|\mathcal{Y}_n[i,t]=\mathcal{Z}[i,t]\Big]\Big]\nonumber\\
    \leq& E\Big[\min_{j\in\mathcal{C}_i}\sum_{t=1}^Tz[j,t]\Big]+\frac{\log |\mathcal{C}_i|}{\eta_i}\\
    &+\eta_i\sum_{t=1}^T\sum_{j\in\mathcal{C}_i}\frac{e^{\eta_i\theta[i,j,t]}}{\sum_{k\in\mathcal{C}_i}e^{\eta_i\theta[i,k,t]}}E\Big[{z[j,t]}^2\Big]\\
    \leq&\min_{j\in\mathcal{C}_i}\sum_{t=1}^Ty_\epsilon[j,t]+\frac{\log |\mathcal{C}_i|}{\eta_i}+\eta_i\sum_{t=1}^T\frac{|\mathcal{C}_i|}{v[i,t]}\nonumber&(\because\mbox{Lemma \ref{lemma:z[j,t]}})\\
    \leq&\min_{j\in\mathcal{C}_i}\sum_{t=1}^Ty_\epsilon[j,t]+\frac{\log D}{\eta_i}+\eta_i\sum_{t=1}^T\frac{D}{v[i,t]}.
\end{align*}
Combining the above inequality with Lemma~\ref{lemma:y_n vs y_e} and we have
\begin{align*}
    &\sum_{t=1}^TE\Big[y_\epsilon[i,t]\Big|m[i,t]=E,\mathcal{Y}_\epsilon[i,t],\mathcal{Z}[i,t-1]\Big]\\
\leq&\min_{j\in\mathcal{C}_i}\sum_{t=1}^Ty_\epsilon[j,t]+\frac{\log D}{\eta_i}+\eta_i\sum_{t=1}^T\frac{D}{v[i,t]}.
\end{align*}
Moreover, since $y_\epsilon[j,t]\in[0,1]\forall j,t$, we clearly have
\begin{align*}
    &\sum_{t=1}^TE\Big[y_\epsilon[i,t]\Big|m[i,t]=U,\mathcal{Y}_\epsilon[i,t],\mathcal{Z}[i,t-1]\Big]\\
    \leq&\min_{j\in\mathcal{C}_i}\sum_{t=1}^Ty_\epsilon[j,t]+T.
\end{align*}

Since $m[i,t]=U$ with probability $\epsilon_i$ and $m[i,t]=E$ with probability $1-\epsilon_i$, we now have
\begin{align*}
    &\sum_{t=1}^TE\Big[y_\epsilon[i,t]\Big|\mathcal{Y}_\epsilon[i,t],\mathcal{Z}[i,t-1]\Big]\\
    \leq&\min_{j\in\mathcal{C}_i}\sum_{t=1}^Ty_\epsilon[j,t]+\epsilon_i T+(1-\epsilon_i)(\frac{\log D}{\eta_i}+\eta_i\sum_{t=1}^T\frac{D}{v[i,t]}),
\end{align*}
and hence 
\begin{align*}
    &\sum_{t=1}^TE\Big[y_\epsilon[i,t]\Big|\mathcal{Y}_\epsilon[i,t]\Big]\\
    \leq&\min_{j\in\mathcal{C}_i}\sum_{t=1}^Ty_\epsilon[j,t]+\epsilon_i T+\frac{\log D}{\eta_i}+\eta_i\sum_{t=1}^T\frac{D}{v[i,t]},
\end{align*}
which proves the first part of the theorem.

Moreover, if we choose $\eta_i=T^{-\frac{L}{L+1}}$ for all $i$ and set $\epsilon_i$ to be 0 if $C_i\subset\mathcal{L}$, and $DT^{-\frac{1}{L+1}}$ otherwise, then $x[i,j,t]\geq\frac{\epsilon_i}{|\mathcal{C}_i|}\geq T^{-\frac{1}{L+1}}$ for all $i,j,t$. Since $v[r,t]=1$ for the root node $r$ and each non-leaf node $i$ is at most $(L-1)-$hops from $r$, we have $v[i,t]\geq T^{-\frac{L-1}{L+1}}$. Putting these into the above inequality and we have

\begin{align*}
    &\sum_{t=1}^TE\Big[y_\epsilon[i,t]\Big|\mathcal{Y}_\epsilon[i,t]\Big]-\min_{j\in\mathcal{C}_i}\sum_{t=1}^Ty_\epsilon[j,t]\\
    \leq& \begin{cases}
    T^{\frac{L}{L+1}}\log D+DT^{\frac{L}{L+1}}, &\textbf{if } C_i\subset\mathcal{L},\\
    DT^{\frac{L}{L+1}}+T^{\frac{L}{L+1}}\log D+DT^{\frac{L}{L+1}} &\textbf{else}.
    \end{cases}
\end{align*}
This completes the proof.
\end{proof}

\end{document}